\documentclass{article}

\PassOptionsToPackage{numbers, compress}{natbib}
\usepackage[preprint]{nips_2018}

\usepackage{graphicx}
\usepackage{amsmath,amssymb,amsfonts,amstext,amsthm,mathrsfs}
\newcommand\myeq{\mathrel{\stackrel{\makebox[0pt]{\mbox{\normalfont\tiny {$\alpha=\frac{\sigma^2_{\max}(\Lb_+)}{2\sigma_{\min}(\Lb_-)}$}}}}{=}}}
\usepackage{thmtools}
\usepackage{thm-restate}
\usepackage{latexsym,epsf,psfrag,epsfig,epstopdf}
\graphicspath{{./figs/}}

\usepackage{dsfont}

\usepackage{algorithm}
\usepackage{algpseudocode}
\usepackage{caption}
\usepackage{subfig}
\usepackage{color} 
\usepackage{enumitem}

\newtheorem{definition}{Definition}

\newtheorem{lemma}{Lemma}
\newtheorem{theorem}{Theorem}
\newtheorem*{assum*}{Assumption}
\newtheorem{remark}{Remark}

\newtheorem{corollary}{Corollary}
\newtheorem{assumption}{Assumption}

\usepackage{xspace}
\newcommand{\rb}{\mathbf{r}}
\newcommand{\y}{\mathbf{y}}

\newcommand{\q}{\mathbf{q}}
\newcommand{\p}{\mathbf{p}}
\newcommand{\e}{\mathbf{e}}
\newcommand{\z}{\mathbf{z}}
\newcommand{\x}{\mathbf{x}}

\newcommand{\Lb}{\mathbf{L}}
\newcommand{\B}{\mathbf{B}}
\newcommand{\A}{\mathbf{A}}
\newcommand{\X}{\mathbf{X}}
\newcommand{\M}{\mathbf{M}}
\newcommand{\Q}{\mathbf{Q}}
\newcommand{\hx}{\hat{\x_T}}
\newcommand{\G}{{\mathbf{G}}}
\newcommand{\W}{\mathbf{W}}
\newcommand{\I}{\mathbf{I}}

\newcommand{\wb}{\mathbf{w}}

\newcommand{\RR}{\mathds{R}}

\usepackage{microtype}
\usepackage{graphicx}
\usepackage{booktabs} % for professional tables

\usepackage{hyperref}

\title{Robust Decentralized Learning Using ADMM with Unreliable Agents}

\author{
  Qunwei Li\\%\thanks{Use footnote for providing further information about author (webpage, alternative address)---\emph{not} for acknowledging funding agencies.} \\
  Syracuse University\\
  \texttt{qli33@syr.edu} \\
  \And
  Bhavya Kailkhura\\
  Lawrence Livermore National Lab\\
  \texttt{kailkhura1@llnl.gov} \\
  \And
  Ryan Goldhahn \\
  Lawrence Livermore National Lab\\
  \texttt{goldhahn1@llnl.gov} \\ 
  \And
  Priyadip Ray\\
  Lawrence Livermore National Lab\\
  \texttt{ray34@llnl.gov} \\ 
  \And
  Pramod K.~Varshney\\
  Syracuse University\\
  \texttt{varshney@syr.edu} \\
}

\begin{document}
% \nipsfinalcopy is no longer used

\maketitle

\begin{abstract}
	Many machine learning problems can be formulated as consensus optimization problems which can be solved efficiently via a cooperative multi-agent system. However, the agents in the system can be unreliable due to a variety of reasons: noise, faults and attacks. Providing erroneous updates leads the optimization process in a wrong direction, and degrades the performance of distributed machine learning algorithms. This paper considers the problem of decentralized learning using ADMM in the presence of unreliable agents. First, we rigorously analyze the effect of erroneous updates (in ADMM learning iterations) on the convergence behavior of multi-agent system. We show that the algorithm linearly converges to a neighborhood of the optimal solution under certain conditions and characterize the neighborhood size analytically. Next, we provide guidelines for network design to achieve a faster convergence. We also provide conditions on the erroneous updates for exact convergence to the optimal solution. Finally, to mitigate the influence of unreliable agents, we propose \textsf{ROAD}, a robust variant of ADMM, and show its resilience to unreliable agents with an exact convergence to the optimum.
%at a rate of $\mathcal{O}(1/T)$ for convex objective functions.
\end{abstract}

\section{Introduction}\label{sec: intro}

Many machine learning and statistics problems fit into the general framework where a finite-sum of functions is to be optimized.
In general, the problem is formulated as
\begin{equation}
\label{f_sum}
\min_{\x\in \mathds{R}^N}\;f(\x),\quad f(\x)=\sum_{i=1}^D f_i(\x).
\end{equation}
The problem structure in~\eqref{f_sum} is applicable to collaborative autonomous inference in statistics, distributed cooperative control of unmanned vehicles in control theory, and training of models (such as, support vector machines, deep neural networks, etc.) in machine learning. 
Due to the emergence of the big data era and associated sizes of datasets, solving problem~\eqref{f_sum} at a single node (or agent) is often infeasible.%, as storing and processing the entire dataset at a single node become infeasible. 
 This gives rise to the decentralized optimization setting~\cite{federated,boyd_admm}, in which the training data for the problem is stored and processed % in a distributed fashion 
 across a number of interconnected nodes and the optimization problem is solved collectively by the cluster of nodes. The decentralized learning system can be implemented on an arbitrarily connected network of computational nodes that solves \eqref{f_sum} by treating it as a consensus optimization problem. % such that the nodes provide one common solution.
There exist several decentralized optimization methods for solving~\eqref{f_sum}, including belief propagation~\cite{bp}, distributed subgradient descent algorithms~\cite{nedic}, dual averaging methods~\cite{duchi}, and the alternating direction method of multipliers (ADMM)~\cite{boyd_admm}. Among these, ADMM has drawn significant attention, as it is well suited for decentralized optimization and demonstrates fast convergence in many applications, such as online learning, decentralized collaborative learning, neural network training, and so on~\cite{Hong2017,admm-aistats,admm-icml}. 
%{{The ADMM solver in these applications involves two basic steps: (i) a communication step for exchanging information only among single-hop neighbors; and, (ii) an update step for updating the local solution at each agent. By alternating between the two, local iterates eventually converge to the global solution. Performance of the applications heavily depends on the convergence (at least acceptable accuracy) of ADMM. Therefore, an immense amount of effort has been put in to establish convergence rates of ADMM in different scenarios~\cite{shi2014linear}.

However, most of these past works assume an ideal system where %communication among agents is noiseless and 
updates are not erroneous. This assumption is very restrictive and rarely met in practice which limits the applicability of these results. % in real-world situations. 
Note that due to the decentralized nature of the systems considered, computation over federated machines induces a higher risk of unreliability because of communication noise, crash failure, and adversarial attacks. Therefore, the design and analysis of decentralized optimization algorithms in the presence of these practical challenges is of utmost importance.
{{A systematic convergence analysis of ADMM in the presence of unreliable agents has been void for a long time. The reason is that unreliable agents have a large degree of freedom %and can falsify any algorithm parameters 
without abiding to an error model and this makes the convergence analysis significantly more challenging as existing proof techniques used in studying the convergence of ADMM do not directly apply. }}

{\textbf{Related work.} {Although, the problem of design and analysis of ADMM with unreliable agents has not been considered in the past, a related research direction is: inexact consensus ADMM~\cite{chang2015multi,bnouhachem2013inexact, yuan2005improvement, xiao2012inexact, ng2011inexact, gu2014inexact}.
The inexactness in ADMM can be categorized as of two different types. Type $1$ assumes that there are errors that can occur in an intermediate step of proximal mapping in each ADMM iteration. Type $2$ replaces the computationally complex calculation in each ADMM iteration by a proximity operator that can be computed more easily, and hence inexactness occurs. Error in inexact ADMM is induced implicitly in intermediate proximal mapping steps and, thus, has a specific restrictive and bounded form with amenable properties for convergence analysis (such as, it converges to zero). These assumptions are very limited in their ability to model unreliability in updates, and are different from what we have studied in our paper.
Furthermore, since the proof techniques for the convergence analysis of inexact ADMMs are designed on an algorithm-by-algorithm basis with restrictive assumptions on error, it lacks a unified framework to analyze the convergence problem of ADMM with an arbitrary error model (of utmost importance to cyber physical security and noisy communication channel scenarios).}}

{\textbf{Contributions.} { This paper proposes a unified framework to study the convergence analysis of decentralized ADMM algorithms in the presence of an arbitrary error model\footnote{Note that, the results in inexact ADMM literature~\cite{chang2015multi,bnouhachem2013inexact, yuan2005improvement, xiao2012inexact, ng2011inexact, gu2014inexact} can be seen as a special cases of our analysis.}. %More specifically, we focus on the Byzantine falsification~\citep{lamport_byzantine, bhavya_consensus, kailkhura2017byzantine} where Byzantines have a large degree of freedom and can falsify any algorithm parameters without abiding to an error model. 
We consider a general error model where an unreliable agent $i$ adds an arbitrary error term $\e_i^k$ to its state value $\x_i^k$ at each time step $k$. The error first contaminates $\x_i^k$ and the resulting output $\x_i^k+\e_i^k$ is broadcast to the neighboring agents. First, we provide a comprehensive convergence analysis both for convex (and strongly convex) cost functions. Next, we show that ADMM converges to a neighborhood of the optimal solution if certain conditions involving the network topology, the properties of the objective function, and algorithm parameters, are satisfied. Guidelines are developed for network structure design and algorithm parameter optimization to achieve faster convergence. 	We also give several conditions on the errors such that exact convergence to the optimum can be achieved, instead to the neighborhood of the optimum. Finally, to mitigate the effect of unreliable agents, a robust variant of ADMM, referred to as \textsf{ROAD}, is proposed. We show that \textsf{ROAD} achieves exact convergence to the optimum with a rate of $\mathcal{O}(1/T)$ for convex cost functions.

\section{Problem Formulation}\label{sec: Pre}
%\subsection{ADMM}
%Generally speaking, ADMM can be employed to solve the convex optimization
%problem formulated in the form of
%\begin{align}
%\begin{split}
%&\min \limits_{x} f(x)\\
%& s.t.\ Ax=b,
%\end{split}
%\end{align}
%with variable $x\in \mathbb{R}^{LN}$, where $A\in \mathbb{R}^{M\times LN}$, $f: \mathbb{R}^{LN} \rightarrow \mathbb{R}$ is convex, and $Ax=b$ is a linear constraint of $x$. ADMM solves an alternating sequence of subproblems one at a time, and it iterates to converge as long as a saddle
%point exists. 
%
%The consensus optimization problem in a collaborative autonomous network of $L$ agents can be formulated as
%\begin{align}\label{pro1}
%\min \limits_{\tilde{x}} \sum\limits_{i=1}^L f_i(\tilde{x})
%\end{align}
%over a common optimization variable $\tilde{x}$, where $f_i(\tilde{x}): \mathbb{R}^N \rightarrow \mathbb{R}$ is the local objective function of agent $i\in \{1,\ldots, L\}$.  

\subsection{Decentralized Learning with ADMM}

Consider a network consisting of $D$
agents bidirectionally connected with $E$ edges.
We can describe the network as a symmetric directed graph $\mathcal{G}_d=\{\mathcal{V,A}\}$, where $\mathcal{V}$
is the set of vertices and $\mathcal{A}$ is the set of
arcs with $|\mathcal{A}|=2E$.
In a  distributed setup, a connected network of agents collaboratively minimize the sum of their local loss functions  over a common optimization variable. Each agent generates local updates individually and communicates with its neighbors to reach a network-wide common minimizer.
The decentralized learning problem, can be formulated as follows
\begin{equation}\label{consensus_form}
\min\limits_{\{\x_i\},\{\mathbf y_{ij}\}}\sum\limits_{i=1}^Df_i(\x_i),
\quad s.t.\ \x_i=\mathbf y_{ij},\ \x_j=\mathbf y_{ij}, \ \forall(i,j)\in \mathcal{A},
\end{equation}
where $\x_i \in \RR^N$ is the local optimization variable at agent $i$ and $\mathbf y_{ij}\in \RR^N$ is an auxiliary variable imposing the consensus
constraint on neighboring agents $i$ and $j$. 
Defining $\x\in \mathds {R}^{DN}$ as a vector concatenating all $\x_i$, $\mathbf y\in \mathbb{R}^{2EN}$ as a vector concatenating all $\mathbf y_{ij}$, \eqref{consensus_form} is written in a matrix form as
\begin{equation}
\min\limits_{\x,\mathbf y}f(\x)+g(\mathbf y),\quad
s.t. \ \A\x+\B\mathbf y=0,
\end{equation}
where $f(\x)=\sum\limits_{i=1}^Df_i(\x_i)$ and $g(\mathbf y)=0$. Here $\A=[\A_1;\A_2]; \A_1,\A_2\in \mathds{R}^{2EN\times LN}$ are both composed of $2E\times D$ blocks of $N\times N$ matrices. If
$(i,j)\in \mathcal{A}$ and $\mathbf y_{ij}$ is the $q$th block of $\mathbf y$, then the $(q,i)$th block of $\A_1$ and the $(q,j)$th block of $\A_2$ are $N\times N$ identity matrices $\mathbf I_N$; otherwise
the corresponding blocks are $N\times N$ zero matrices $\mathbf 0_N$. Also,we
have $\B=[-\mathbf I_{2EN};-\mathbf I_{2EN}]$ with $\mathbf I_{2EN}$ being a $2EN\times 2EN$
identity matrix.
Define the matrices: $\M_+=\A_1^T+\A_2^T$ and $\M_-=\A_1^T-\A_2^T$. Let $\W\in \mathds{R}^{DN\times DN}$ be a block diagonal matrix with its $(i,i)$th block being the degree of agent $i$ multiplying $\I_N$ and other blocks being $\mathbf 0_N$, $\Lb_+=\frac{1}{2} \M_+\M_+^T$, $\Lb_-=\frac{1}{2}\M_-\M_-^T$, and we know $\W=\frac{1}{2}(\Lb_++\Lb_-)$. These matrices are related
to the underlying network topology.%For undirected graph $\mathcal{G}_u$, $\M_+$ and $\M_-$ are the extended unoriented and oriented incidence matrices, respectively; $\Lb_+$ and $\Lb_-$ are the extended signless and signed Laplacian matrices, respectively; and $\W$ is the extended degree matrix. By “extended”, we mean replacing every $1$ by $\I_N$, $-1$ by $-\I_N$, and $0$ by $\mathbf{0}_N$ in the original definitions of these matrices.

\subsection{Decentralized ADMM with Unreliable Agents}

The iterative updates of the decentralized ADMM algorithm are given by \cite{shi2014linear} as
\begin{equation}\label{shi_dec_up}
\begin{split}
&\x-\texttt{update}: \nabla f(\x^{k+1})+\alpha^k+2c\W\x^{k+1}=c\Lb_+\x^k,\\
&\mathbf \alpha-\texttt{update}: \alpha^{k+1}-\alpha^k-c\Lb_-\x^{k+1}=0.
\end{split}
\end{equation}
Note that $\x=[\x_1;\ldots;\x_D]$
where $\x_i\in \RR^N$ is the local update of agent $i$ and $\alpha=[\alpha_1;\ldots;\alpha_D]$ 
where $\alpha_i\in \mathds{R}^N$ is the local Lagrange multiplier of
agent $i$. Recalling the definitions of $\W$, $\Lb_+$ and $\Lb_-$, \eqref{shi_dec_up} results in
the decentralized update of agent $i$ given as follows
\begin{align*}
\begin{split}
&\nabla f_i(\x_i^{k+1})+\alpha_i^k+2c|\mathcal{N}_i|\x_i^{k+1}=c|\mathcal{N}_i|\x_i^k+c\sum\limits_{j\in \mathcal{N}_i} \x_j^k,\\
&\alpha_i^{k+1}=\alpha_i^k+c|\mathcal{N}_i|\x_i^{k+1}-c\sum\limits_{j\in \mathcal{N}_i}\x_j^{k+1},
\end{split}
\end{align*}
where $\mathcal{N}_i$ denotes the set of neighbors of agent $i$. 
%The algorithm is fully decentralized since the updates of $\x_i$ and $\alpha_i$ only rely on local and neighboring information.

In such a setup, we consider the case where a fraction of the agents are unreliable and generate erroneous updates. Assume that the true update is $\x^k$, and the erroneous update is modeled as $\x^k+\e^k$, which is denoted as $\z^k=\x^k+\e^k$. The corresponding algorithm becomes 
\begin{align*}
\begin{split}
&\nabla f_i(\x_i^{k+1})+\alpha_i^k+2c|\mathcal{N}_i|\x_i^{k+1}=c|\mathcal{N}_i|\z_i^k+c\sum\limits_{j\in \mathcal{N}_i} \z_j^k,\\
&\alpha_i^{k+1}=\alpha_i^k+c|\mathcal{N}_i|\z_i^{k+1}-c\sum\limits_{j\in \mathcal{N}_i}\z_j^{k+1}.
\end{split}
\end{align*}

For a clearer presentation, we will use the following form of the updates for our analysis
\begin{equation}\label{errorADMM}
	\begin{split}
	&\x-\texttt{update}: \nabla f(\x^{k+1})+\alpha^k+2c\W\x^{k+1}=c\Lb_+\z^k,\\
	&\alpha-\texttt{update}: \alpha^{k+1}-\alpha^k-c\Lb_-\z^{k+1}=0.
	\end{split}
\end{equation}
Compared to \eqref{shi_dec_up}, $\x^k$ is replaced by the erroneous update $\z^k$ in the first step, and $\x^{k+1}$ is replaced by $\z^{k+1}$ in the second step. The convergence analysis of \eqref{errorADMM} is nontrivial and is not a straightforward extension of the analysis with~\eqref{shi_dec_up} in \cite{shi2014linear}. Additionally, the analysis in \cite{shi2014linear} was restricted to strongly convex cost functions. We analyze the problem for both convex and strongly convex cost functions.

\subsection{Problem Assumptions}\label{sec: models}

We provide definitions and assumptions that will be used for the cost functions in our analysis.

\begin{definition}For a differentiable function $f(\x): \RR^{DN} \rightarrow \RR$:
\begin{itemize}
\item $f$ is $v$-strongly convex if $\forall \x,\y \in \mathds{R}^{DN}$, $f(\x)\ge f(\y) + \langle \nabla f(\y), \x-\y \rangle +v\|\x-\y\|^2$.
\item $f$ is $L$-smooth if $\forall \x, \y \in \mathds{R}^{DN}$, $\|\nabla f(\x)-\nabla f(\y)\| \le L \|\x-\y\|$.
\end{itemize}
\end{definition}
% The constant $v$ measures how convex a function $f$ is. In particular, the larger the value of $v$, the more convex $f$ is. 
\begin{assumption}
For a differentiable function $f(\x): \RR^{DN} \rightarrow \RR$:
\begin{itemize}
\item The feasible $\x \in \RR^N$ is bounded as $\|\x\|\le V_1$.
\item The gradient $\nabla f(\x)$ is bounded as $\|\nabla f(\x)\|\le V_2$.
\end{itemize}
\end{assumption}
{{Note that these assumptions are very common in the analysis of first-order optimization methods~\cite{bottou2016optimization}.}}

%!TEX root = main.tex

\section{Convergence Analysis}\label{sec: result}

To effectively present the convergence results\footnote{Proofs of the theoretical analysis are provided in the supplementary material.}, 
we first introduce a few notations. %Let $\Q=\left(\frac{\Lb_-}{2}\right)^{\frac{1}{2}}$. Specifically, let
Let $\Q=\mathbf V\mathbf \Sigma^{\frac{1}{2}}\mathbf  V^T$, where $\frac{\Lb_-}{2}=\mathbf V\mathbf \Sigma \mathbf V^T$ is the singular value decomposition of the positive semidefinite matrix $\frac{\Lb_-}{2}$. We also construct a new auxiliary sequence $\mathbf r^k=\sum \limits_{s=0}^k \Q(\x^s+\e^s)$.  Let $\z^\ast=\x^\ast$, where $\x^\ast$ denotes the optimal solution to the problem.
Define the auxiliary vector $\mathbf q^k$,  matrix $\mathbf p^k$, and matrix $\mathbf G$ as
\begin{align*}
	&\mathbf q^k=
	\begin{bmatrix}
	\mathbf r^k\\ \z^k
	\end{bmatrix}, 
	\mathbf p^k=
	\begin{bmatrix}
	\mathbf r^k\\ \x^k
	\end{bmatrix},
	\mathbf G=\begin{bmatrix}
	c\I & \mathbf 0\\ 
\mathbf	0 &c \Lb_+/2
	\end{bmatrix}.
%	\\	&\mathbf G_2=\begin{bmatrix} c\I-\frac{c\I}{2\sigma_{\max}(\W)}   & \mathbf 0\\ \mathbf	0 & \left( \frac{\I}{2}+\frac{\I-2\W}{4\sigma_{\max}(\W)} \right) c\Lb_+\end{bmatrix}.
\end{align*}
For a positive semidefinite matrix $\mathbf X$, we use $\sigma_{\min}(\X)$ as the nonzero smallest eigenvalue of matrix $\X$ and  $\sigma_{\max}(\X)$ as the nonzero largest eigenvalue in sequel. 
%{\color{red}{Ideally, you should write which is primal variable, dual variable and what is the meaning of *. Furthemore, concatining primal and dual variable and talking about their convergence analysis is ADMM specific...so tell the reader what are you doing and why are you doing it. Due to space constraints, you can ignore these commnets.}}

\subsection{Convex Case}

In this case, we assume convexity for the cost function and analyze the convergence of the ADMM algorithm in the presence of errors.
\begin{theorem}\label{convex}
There exists $\mathbf p=\begin{bmatrix}
\mathbf{r}\\ 
\x^\ast
\end{bmatrix}$ with $\rb=0$ such that
\begin{align}\label{convex_first}
	f(\x^{T})-f(\x^\ast) \le \|\q^{T-1}-\p\|^2_{\mathbf{G}},\; \text{and}
\end{align}
\begin{align}\label{convex_second}
\frac{\sum_{k=1}^T f(\x^k)}{T}-f(\x^\ast) \le \frac{\|\p^0-\p\|^2_\G}{T}+\frac{c}{T}\sum_{k=1}^T \left(\frac{\sigma^2_{\max}(\Lb_+)}{2\sigma_{\min}(\Lb_-)} \|\e^{k}\|^2_2+\langle \e^{k},2\Q\rb^{k})\rangle \right).
\end{align}
\end{theorem}
Theorem \ref{convex} provides the upper bound for the residual of the function value over the iterations, and shows how errors accumulate and affect the convergence of the algorithm. In \eqref{convex_first}, the effect of the errors that occurred before the $T$-th iteration is represented by $\q^{T-1}$, which means that the previous errors have accumulated to impact the current algorithm state. It is observed in \eqref{convex_second} that the averaged function value approaches the neighborhood of the minimum function value in a sub-linear fashion, and the second term on the right hand side of the bound represents the radius of this neighborhood. It also shows that the algorithm converges sub-linearly if after a certain number of iterations, there are no errors in the updates. {Comparing to convergence rate of $\mathcal O(\frac{1}{T})$ with decentralized ADMM for convex programming, e.g., \cite{7870649},  our result is very different. In the presence of errors, the algorithm converges to the neighborhood of the minimizer with a rate of $\mathcal O(\frac{1}{T})$ as well, but the true convergence to the minimizer cannot be guaranteed. The bounds are obtained in the form of $\mathbf G$ norm. Recall the definition of $\mathbf G$, we can see that the structure of the network also plays a role in bounding the residual of the function value. Both the bounds show that a network with smaller $\sigma_{\max}(\Lb_+)$ (which is proportional to the network connectivity) is more resilient to errors.
{{Intuitively, a less connected network can lower the spread of the errors. However, a more connected network has a faster convergence speed. This observation also highlights a potential trade-off between the resilience and the convergence speed. }}
% \begin{remark}\label{remark:1}
% 	Since the volume of the injected errors can be unlimited, the value of $r^k$ is unbounded. Thus, the averaged function value can never reach the minimal function value, even if $k$ approaches infinity. However, if one can design a robust ADMM algorithm such that $r^k$ is bounded, then we will see that the corresponding residual is also bounded as $k$ approaches infinity.
% \end{remark}
% \begin{remark}
% Due to the convexity of the function and using Jensen's inequality, we can obtain that the residual $f(\bar{\x}^{k})-f(\x^\ast)$ with $\bar{\x}^{k}=\frac{1}{k} \sum_{s=0}^{k} \x^s$ is also bounded by the same expression. This means that if $r^k$ can be bounded, the mean of the update $\bar{\x}^{k}$ %lies in a neighborhood of the minimizer $\x^\ast$. Similarly, if the noisy update $\|\z^k\|^2_2$ is bounded, then $\bar{\x}^{k}$ 
% approaches the neighborhood of the minimizer with a rate of $\mathcal O(\frac{1}{k})$. 
% \end{remark}

\subsection{Strongly Convex \& Lipschitz Continuous Case}

We assume that $f(\x)$ is $v$-strongly convex and $L$-smooth, and provide the convergence analysis.

\begin{theorem}\label{lemma6}
	There exists $\mathbf q^\ast=\begin{bmatrix}
	\mathbf r^\ast\\ 
	\x^\ast
	\end{bmatrix}$ such that for the $k$-th iteration,
	\begin{align*}
	\|\mathbf q^{k}-\mathbf q^\ast\|^2_\G \le & \frac{\|\q^{k-1}-\q^\ast\|^2_\G}{1+\delta}+\frac{P\|\e^{k}\|^2_2+\langle \e^{k}, \mathbf{s}\rangle}{1+\delta}
	\end{align*}
    with
   $
    \mathbf{s}=c\Lb_+(\z^{k}-\z^{k-1})+2c\Q(\mathbf r^{k}-\mathbf r^\ast)+2c\W(\x^{k}-\x^\ast),
    $
	where 
	$
	P=\frac{c^2 \delta\lambda_2\sigma^2_{\max}(\W)}{\sigma^2_{\min}(\Q)}+\frac{c^2\delta \lambda_3\sigma^2_{\max}(\Lb_+)}{4}
	$,\; \text{and}
	\begin{align*}
	\delta=\min\left\lbrace\frac{(\lambda_1-1)(\lambda_2-1)\sigma^2_{\min}(\Q)\sigma^2_{\min}(\Lb_+)}{\lambda_1\lambda_2\sigma^2_{\max}(\Lb_+)},\right. \left.\frac{4v(\lambda_2-1)(\lambda_3-1)\sigma^2_{\min}(\Q)}{\lambda_1\lambda_2(\lambda_3-1)L^2+c^2\lambda_3(\lambda_2-1)\sigma^2_{\max}(\Lb_+)\sigma^2_{\min}(\Q)}\right\rbrace
	\end{align*}
	with quantities $\lambda_1$, $\lambda_2$, and $\lambda_3$ being greater than 1.
\end{theorem}

Theorem~\ref{lemma6} shows that the sequence $\|\q^{k}-\q^\ast\|^2_\G$ converges linearly with a rate of $\frac{1}{1+\delta}$ if after a certain number of iterations, there are no data-falsification errors in the updates. Then, it can be easily shown that the sequence $\z^{k}$ or $\x^{k}$ converges to the minimizer. However, if the errors persist in the updates, this theorem shows how the errors are accumulated after each iteration. As a general result, one can further optimize over $\lambda_1$, $\lambda_2$, and $\lambda_3$ to obtain maximal $\delta$ and minimal $P$ to achieve fastest convergence and least impact from the errors.

\begin{theorem}\label{thm1}
	Choose $
	0<\beta \le \frac{b(1+\delta)\sigma^2_{\min}(\Lb_+)\left(1-\frac{1}{\lambda_4}\right)}{4b\sigma^2_{\min}(\Lb_+)\left(1-\frac{1}{\lambda_4}\right)+16\sigma^2_{\max}(\W)}
	$
	where $b>0$ and $\lambda_4>1$, 
	then
	\begin{align*}
	\|\z^{k}-\z^\ast\|^2_2\le  B^{k} \left( {A}+\sum\limits^{k}_{s=1} B^{-s}C\|\e^{s}\|^2_2\right)
	\end{align*}
	where 
    $
    {A}=\|\z^{0}-\z^\ast\|^2_2+A_2\|\mathbf r^{0}-\mathbf r^\ast\|^2_2
    $
    with
    $
    A_2=\frac{4}{(1+4\beta)\sigma^2_{\max}(\Lb_+)},
	$
    and
    $
B=\frac{(1+4\beta)\sigma^2_{\max}(\Lb_+)}{(1-b)(1+\delta-4\beta)\sigma^2_{\min}(\Lb_+)},
$
$
C=\frac{4P+2/\beta}{c^2(1-b)(1+\delta-4\beta)\sigma^2_{\min}(\Lb_+)}+\frac{b(\lambda_4-1)}{1-b}.
$
\end{theorem}

Theorem~\ref{thm1} presents a general convergence result for ADMM for decentralized consensus optimization with errors, and indicates that the erroneous update $\z^{k}$  approaches the neighborhood of the minimizer in a linear fashion. The radius of the neighborhood is given as $B^{k}\sum\limits^{k}_{s=1} B^{-s}C\|\e^{s}\|^2_2$.  Note that $B$ is not guaranteed to be less than $1$. This is very different from the convergence result of ADMM for decentralized consensus optimization \cite{shi2014linear}, which can guarantee that the update converges to the minimizer linearly fast and the corresponding rate is less than $1$. Additionally, if $\sigma^2_{\max}(\Lb_+) >> \sigma^2_{\min}(\Lb_-)$, and it ends up with $B$ being greater than $1$, then the algorithm will not converge at all. 

Thus, the first problem that follows is to guarantee that $B$ is within the range $(0,1)$, and the second one is to minimize the radius of the neighborhood by minimizing $C$.
Accordingly, we optimize over the variables that appeared in the above theorems and the algorithm parameter $c$, and give the convergence result with $B\in (0,1)$.

\begin{theorem}\label{convergence}
	If $b$ and $\lambda_2$ can be chosen, such that
	\begin{align}\label{net}
		(1-b)(1+\delta)\sigma^2_{\min}(\Lb_+)>\sigma^2_{\max}(\Lb_+) 
	\end{align}
	with 
	$
	\delta=\frac{(\lambda_2-1)}{\lambda_2}\frac{2v\sigma^2_{\min}(\Q)\sigma^2_{\min}(\Lb_+)}{L^2\sigma^2_{\min}(\Lb_+)+2v\sigma^2_{\max}(\Lb_+)},
	$
 	then the ADMM algorithm with a parameter $c=\sqrt[]{\frac{\lambda_1\lambda_2(\lambda_3-1)L^2}{\lambda_3(\lambda_2-1)\sigma^2_{\max}(\Lb_+)\sigma^2_{\min}(\Q)}}$ converges linearly with a rate of $B\in(0,1)$, to the neighborhood of the minimizer 
%     with
% 	\begin{align*}
% 	C=\frac{\frac{4\delta \lambda_2 \sigma^2_{\max}(\W)}{\sigma^2_{\min}(\Q)}+\sigma^2_{\max}(\Lb_+)\left(\sqrt[]{\delta}+\sqrt[]{\frac{2(\lambda_2-1)\sigma^2_{\min}(\Q)}{\beta \lambda_1\lambda_2L^2}}\right)^2}{(1-b)(1+\delta)(1+\delta-4\beta)\sigma^2_{\min}(\Lb_+)}+\frac{b(\lambda_4-1)}{1-b}
% 	\end{align*}
	where
	$
	\lambda_1=1+\frac{2v\sigma^2_{\max}(\Lb_+)}{L^2\sigma^2_{\min}(\Lb_+)},
	$
	$
	\lambda_3=1+\sqrt[]{\frac{L^2\sigma^2_{\min}(\Lb_+)+2v\sigma^2_{\max}(\Lb_+)}{\beta\lambda_1L^2v\sigma^2_{\min}(\Lb_+)}}
	$
	and
	\begin{align*}
	0<\beta\le \min\left\lbrace\frac{b(1+\delta)\sigma^2_{\min}(\Lb_+)\left(1-\frac{1}{\lambda_4}\right)}{4b\sigma^2_{\min}(\Lb_+)\left(1-\frac{1}{\lambda_4}\right)+16\sigma^2_{\max}(\W)},\right.
	\left. \frac{(1-b)(1+\delta)\sigma^2_{\min}(\Lb_+)-\sigma^2_{\max}(\Lb_+)}{4\sigma^2_{\max}(\Lb_+)+4(1-b)\sigma^2_{\min}(\Lb_+)}\right\rbrace.
	\end{align*}
	
\end{theorem}

Theorem~\ref{convergence} provides an optimal set of choices of variables and the algorithm parameter such that $B \in (0,1)$ and $C$ is minimized in Theorem~\ref{thm1}. 
Recalling condition \eqref{net}, it is equivalent to
\begin{align}\label{condition_strong_convex}
\frac{\sigma^2_{\min}(\Lb_+)}{\sigma^2_{\max}(\Lb_+)} > \frac{4v}{\sqrt{(L^2+2v)^2+16v^2\frac{\lambda_2-1}{\lambda_2}\sigma^2_{\min}(\Q)}-L^2+2v}.
\end{align}
As the only condition for the convergence, we show in our experiments that it can be easily satisfied.

% \begin{remark}
% The value of $\frac{\sigma^2_{\min}(\Lb_+)}{\sigma^2_{\max}(\Lb_+)} $, which corresponds to the network structure, has to be greater than a certain threshold such that $B\in(0,1)$ can be achieved. This shows that a decentralized network with a random structure may not converge at all to the neighborhood of the minimizer, when the ADMM algorithm is implemented with erroneous updates. 
% \end{remark}
% \begin{remark}
% Considering $\frac{\lambda_2-1}{\lambda_2}\sigma^2_{\min}(\Q)$ as the only variable on the right hand side of inequality \eqref{condition_strong_convex}, it is upper bounded by $\frac{4v}{(\sqrt{2}-1)L^2+(2\sqrt{2}+2)v}$. Thus, if we can design a network such that its corresponding value of $\frac{\sigma^2_{\min}(\Lb_+)}{\sigma^2_{\max}(\Lb_+)} $ is greater than this bound, we can ensure that the decentralized ADMM algorithm can converge to the neighborhood of the minimizer.
% \end{remark}
% \begin{remark}
% The right hand side of the above expression depends on the geometric properties of the cost function. There exists a certain class of cost functions (e.g., $v$ is small), such that the value of the right hand side can be lowered, compared with other cost functions. Thus, these functions allow for a more flexible network structure design such that a linear convergence rate can be achieved.
% \end{remark}

\begin{remark}
The value of $\frac{\sigma^2_{\min}(\Lb_+)}{\sigma^2_{\max}(\Lb_+)} $, which corresponds to the network structure, has to be greater than a certain threshold such that $B\in(0,1)$ can be achieved. This shows that a decentralized network with a random structure may not converge at all to the neighborhood of the minimizer, in the presence of errors in iteration. 
\end{remark}
\begin{remark}
The right hand side of inequality \eqref{condition_strong_convex} is upper bounded by $\frac{4v}{(\sqrt{2}-1)L^2+(2\sqrt{2}+2)v}$, which depends on the geometric properties of the cost function. There exists a certain class of cost functions (e.g., $v$ is small, $L$ is large), such that a more flexible network structure design is allowed for a linear convergence to the neighborhood of the minimizer.
\end{remark}

\begin{corollary}\label{strongly_convex_corollary}
When \eqref{condition_strong_convex} is satisfied, the first condition below achieves linear convergence to the neighborhood of the minimizer with a radius of $\frac{Ce}{1-B}$, and either of the last two conditions guarantees linear convergence to the minimizer
\begin{itemize}
\item 
 $\|\e^{k-1}\|^2_2 \le e$
\item 
$\|\e^k\|^2_2$ decreases linearly at a rate $R$ such that $0<R< B$
\item 
$ C\|\e^{k}\|^2_2 \le B(A_1-A_2)\|\mathbf r^{k-1}-\mathbf r^\ast\|^2_2$ with $A_1=\frac{4}{(1-b)\sigma^2_{\min}(\Lb_+)}$
\end{itemize}
\end{corollary}

The first result in Corollary \ref{strongly_convex_corollary} simply states that if the error at every iteration is bounded, then the algorithm will approach the bounded neighborhood of the minimizer, and the second result states that if the error in the update decays faster than the distance between the update and the minimizer $\|\z^k-\z^\ast\|^2_2$, then the algorithm will reach the minimizer at a linear rate. The third result provides a much more general condition for convergence to the minimizer, %Note that $C$, $B$, $A_1$, and $A_2$ are fixed, and the condition \eqref{conv_condition} relates the current error to all the previously accumulated errors. The error impacts the performance of decentralized ADMM in two different ways. First, the error at an individual local agent makes the local update deviate from the true update. Second, the local error can propagate over the network through update exchange between neighboring agents, thus impacting the update precision of the agents later in the network. Hence, the errors that occurred before the current iteration can propagate and get accumulated over the network. At this point, Theorem \ref{thm4} 
which gives an upper bound for the current error based on the past errors, such that the network can tolerate the accumulated errors and the convergence to the minimizer can still be guaranteed.

%!TEX root = main.tex
%\shrk
\section{Robust Decentralized ADMM Algorithm (\textsf{ROAD})}
\label{sec:implementation}
%\shrk
Based upon insights provided by our theoretical results in Section~\ref{sec: result}, we investigate the design of the robust ADMM algorithm which can tolerate the errors in the ADMM updates. We focus on the scenario where a fraction of the agents generate erroneous updates. The remaining agents in the network follow the protocol and generate true updates, which are referred to as reliable agents\footnote{We also assume that reliable neighbors are in a majority for each agent $i$ in the network.} in this paper. We refer to our proposed robust ADMM algorithm as ``\textsf{ROAD}'' (Algorithm $1$).
\begin{algorithm}\label{alg}
	\caption{\textsf{ROAD}($\x^0,c,\alpha^0,T,U$)}
	\begin{algorithmic}[1]
		\Function{$f=\sum\limits_{i=1}^D f_i(\x)$}{}      	
		\State Initialization: $\x^0=0$, $c$, $\alpha^0=0$, $T$, $U$
		
		\For{$k = 1$ to $T$}
		\State For the node $i$ :
		\If {$\sum_{t=1}^k\|\x_i^t-\x_j^t\|>U$, $j\in \mathcal{N}_i$,}
		\State Replace $\x_j^k$ with $\x_i^k$ {{in current update \eqref{shi_dec_up}}}
        \Else
        \State Use $\x_j^k$ {{in current update \eqref{shi_dec_up}}}        
		\EndIf
		\EndFor
		\State Output $\x^T$
		\EndFunction
	\end{algorithmic}
\end{algorithm}

{{To explain the idea behind ROAD, let us define two crucial variables used in the algorithm: $I)$ deviation statistics $Z(k)=\sum_{t=1}^k\|\Q\z^t\|$, and $II)$ threshold {$U=\left( \sigma_{\max}(\Lb_+)V^2_1+\frac{2V_2^2}{\sigma_{\min}(\Lb_-)c^2}+4\right)/2\sqrt{2}$}. 
The deviation statistics accumulates agents' update deviation from each other over ADMM iterations. 
Next, we obtain an upper bound on the deviation statistics for the error-free case. Specifically, if there were no errors in the updates from the neighbors, we show in Lemma $8$ (in supplementary materials) that $Z(k) \le U/\sqrt{2}$. This upper bound $U$ serves as a threshold to identify unreliable agents. Note that $Z(k)=\frac{1}{\sqrt[]{2}}\sum_{t=1}^k{\sum_{(i,j) \in \mathcal{V}}\|\z_i^k-\z_j^k\|}$, thus, we have $\frac{1}{\sqrt[]{2}}\sum_{t=1}^k{\|\z_i^k-\z_j^k\|}\leq Z(k)\le U/\sqrt{2},\;\forall (i,j)\in \mathcal{V}$. Inspired by this relationship, each agent $i$ maintains the local deviation statistics $\sum_{t=1}^k{\|\z_i^k-\z_j^k\|}$ for every neighboring agent $j \in \mathcal{N}_i$ and compares it with the threshold $U$ to identify if neighboring agent $j$ is providing erroneous updates. For a reliable node $j$, the statistic $\sum_{t=1}^k\|\z_i^t-\z_j^t\|$ will not exceed the threshold $U$. If the statistic $\sum_{t=1}^k\|\z_i^t-\z_j^t\|$ exceeds the threshold $U$, the neighboring agent $j$ is labeled as unreliable and its update is not be used by agent $i$. To avoid network disconnection in the case of unreliable neighbors, the link $\{i,j\}$ would not be cut off, however, the update from $j$ will be replaced by node $i$'s own value. Next, we show in Theorem \ref{radmm} that the proposed \textsf{ROAD} algorithm converges to the optimum at a rate of $\mathcal O(1/T)$.

\begin{theorem}\label{radmm}
For convex function $f(\x)$, there exists $\mathbf p=\begin{bmatrix}
\mathbf{r}\\ 
\x^\ast
\end{bmatrix}$ with $\rb=0$, and \textsf{ROAD} provides
\begin{align}\label{radmm_upper}
	f(\hat{\x}_T)-f(\x^\ast) \le \frac{1}{T}\left( \|\p^0-\p\|^2_\G+8c\frac{\sigma^2_{\max}(\Lb_+)}{\sigma^2_{\min}(\Lb_-)}E^2U^2\right)
\end{align}
where $\hat{\x}_T=\sum_{k=1}^T \x^k/T$, and
$U=\left( \sigma_{\max}(\Lb_+)V^2_1+\frac{2V_2^2}{\sigma_{\min}(\Lb_-)c^2}+4\right)/2\sqrt{2}$.
\end{theorem}
%The proposed algorithm achieves the same sub-linear convergence rate of $\mathcal O(1/T)$ when there is no error, whereas an additional term of $\frac{16c\sigma_{\max}^2(\Lb_+)}{\sigma^2_{\min}(\Lb_-)}E^2U^2$ is introduced. This is because that our algorithm can not identify the unreliable nodes if they add limited errors to the updates. However, the statistic $\sum_{t=1}^k\|\z_i^t-\z_j^t\|$ increases over iteration time and it is bounded, which means $\|\z_i^t-\z_j^t\|\rightarrow 0, \forall \{i,j\}\in \mathcal{V}$ as $k\rightarrow \infty$ and the error injected by unreliable agents is forced to approach 0 by nature. This then connects with the results in Theorem \ref{convex}, Theorem \ref{lemma6} and Corollary \ref{strongly_convex_corollary} that if the error decreases to 0, convergence to the optima can be reached.
{{Theorem $5$ shows that the \textsf{ROAD} achieves a sub-linear convergence rate of $\mathcal O(1/T)$. Note that to account for the thresholding operation in ROAD, the upper bound in $(10)$ introduces an additional term $8c\frac{\sigma^2_{\max}(\Lb_+)}{\sigma^2_{\min}(\Lb_-)}E^2U^2$.}} 
%Further, the local deviation statistics $\sum_{t=1}^k\|\z_i^t-\z_j^t\|$ is non-decreasing in $t$ time and is bounded, which implies $\|\z_i^k-\z_j^k\|\rightarrow 0, \forall \{i,j\}\in \mathcal{V}$ as $k\rightarrow \infty$ and the error injected by unreliable agents is forced to approach $0$ by nature. 
{{ROAD still falls under the formulation in (4) %(with reduced errors) 
and follows the general analysis framework considered in Section 3. Thus, Theorem~\ref{radmm} also connects with the results in Theorem \ref{convex}, Theorem \ref{lemma6} and Corollary \ref{strongly_convex_corollary}. % that if the error approaches $0$ {\color{blue}we should say something of why approaching 0, previously i did}, convergence to the optima can be reached. 
In the next section, we will also show empirically that employing the algorithmic parameter $c$ derived in Theorem $4$ accelerates the convergence rate of ROAD.}}
%!TEX root = main.tex
%\shrk
\section{Experiments}\label{sec:exp}
%\shrk
 In this section, we use \textsf{ROAD} to solve two different decentralized consensus optimization problems
 %. A decentralized network 
 with $D=10$ agents.
 %is employed to perform the optimization task. 
We provide the network topology for the experiments in supplementary materials (Figure $2$). We assume that there are $3$ unreliable agents (chosen randomly) in the network. Unreliable agents introduce errors in their updates by adding Gaussian noise\footnote{{{Note that our theoretical analysis and the proposed mitigation scheme (\textsf{ROAD}) does not assume the error to be of any parametric structure and are applicable to any arbitrary type of error.}}} with mean $\mu_b$ and variance $\sigma_b^2$.

%\begin{figure}[h]
%	\label{fig:topol}
%	\centering
%	\includegraphics[width=0.48\textwidth]{./plots/topol.png}
%	\caption{Decentralized network topology.}
%\end{figure}
\begin{figure}%
	\centering
	\subfloat[]{{\includegraphics[width=6cm]{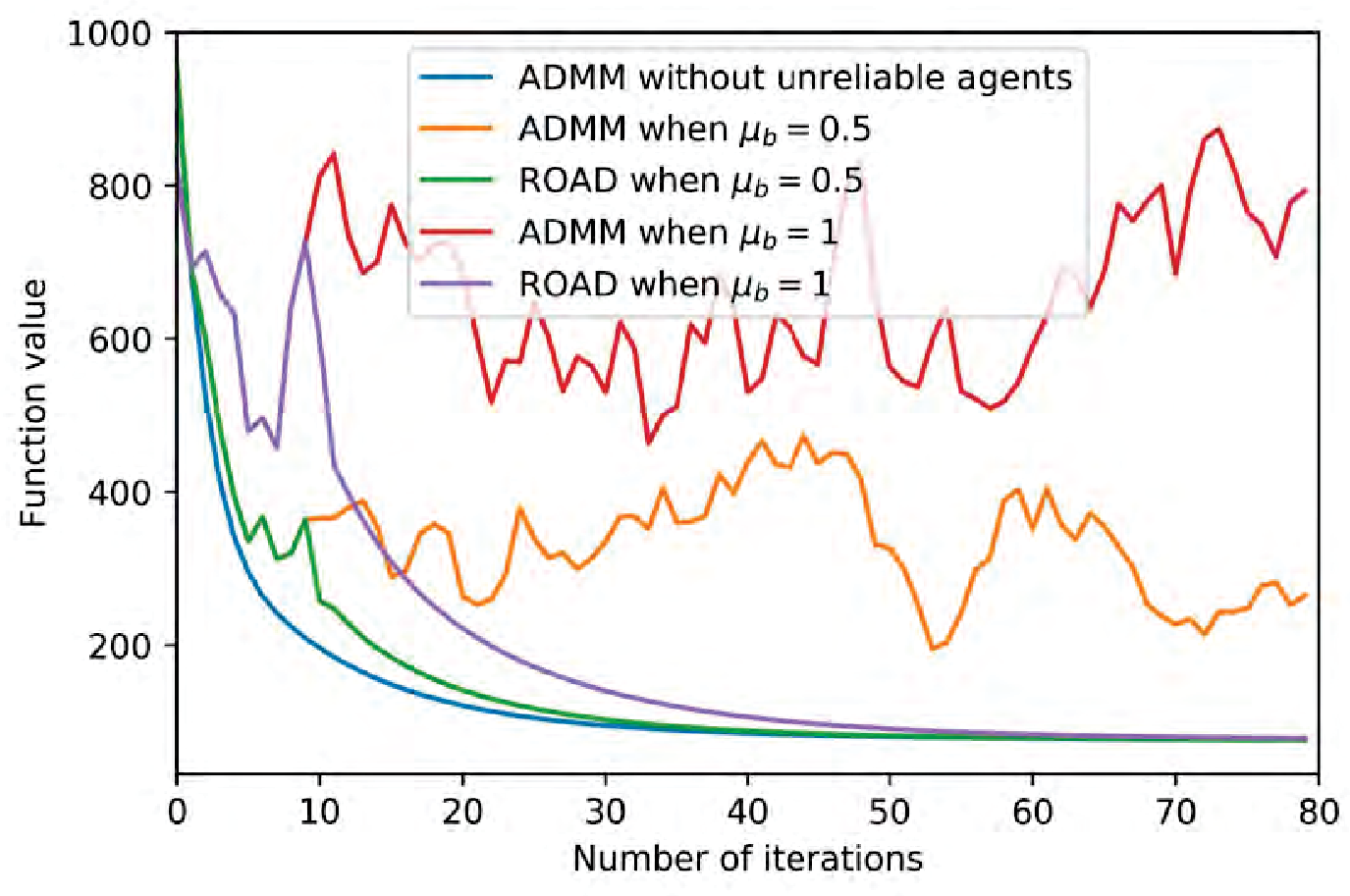} }}%
	\qquad
	\subfloat[]{{\includegraphics[width=6cm]{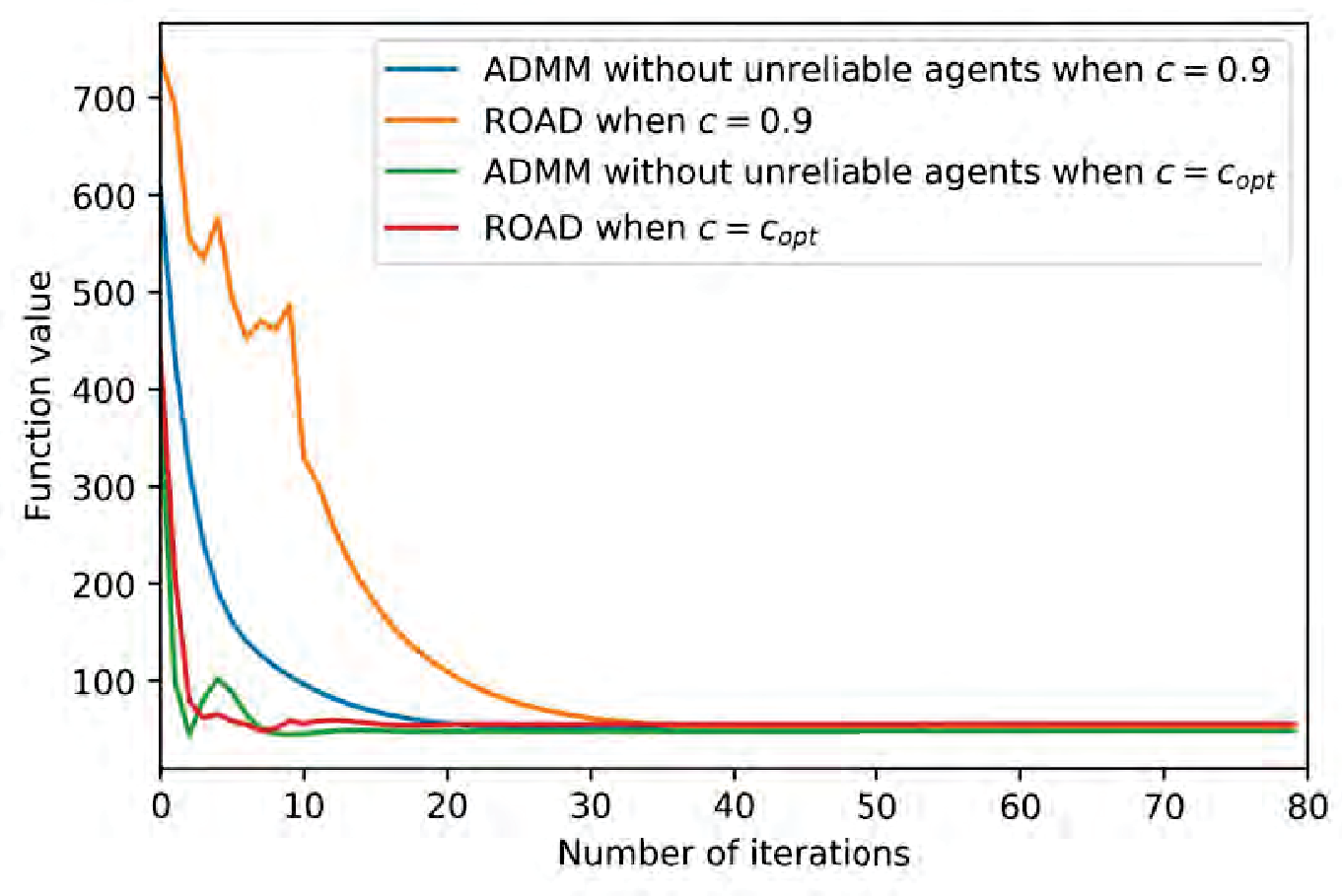} }}%
	\caption{(a) Performance comparison with different noise intensities. (b) Performance comparison with different choices of algorithm parameter. }%
	\label{fig1}%
\end{figure}
\begin{figure}
	\centering
	\subfloat[]{{\includegraphics[width=6cm]{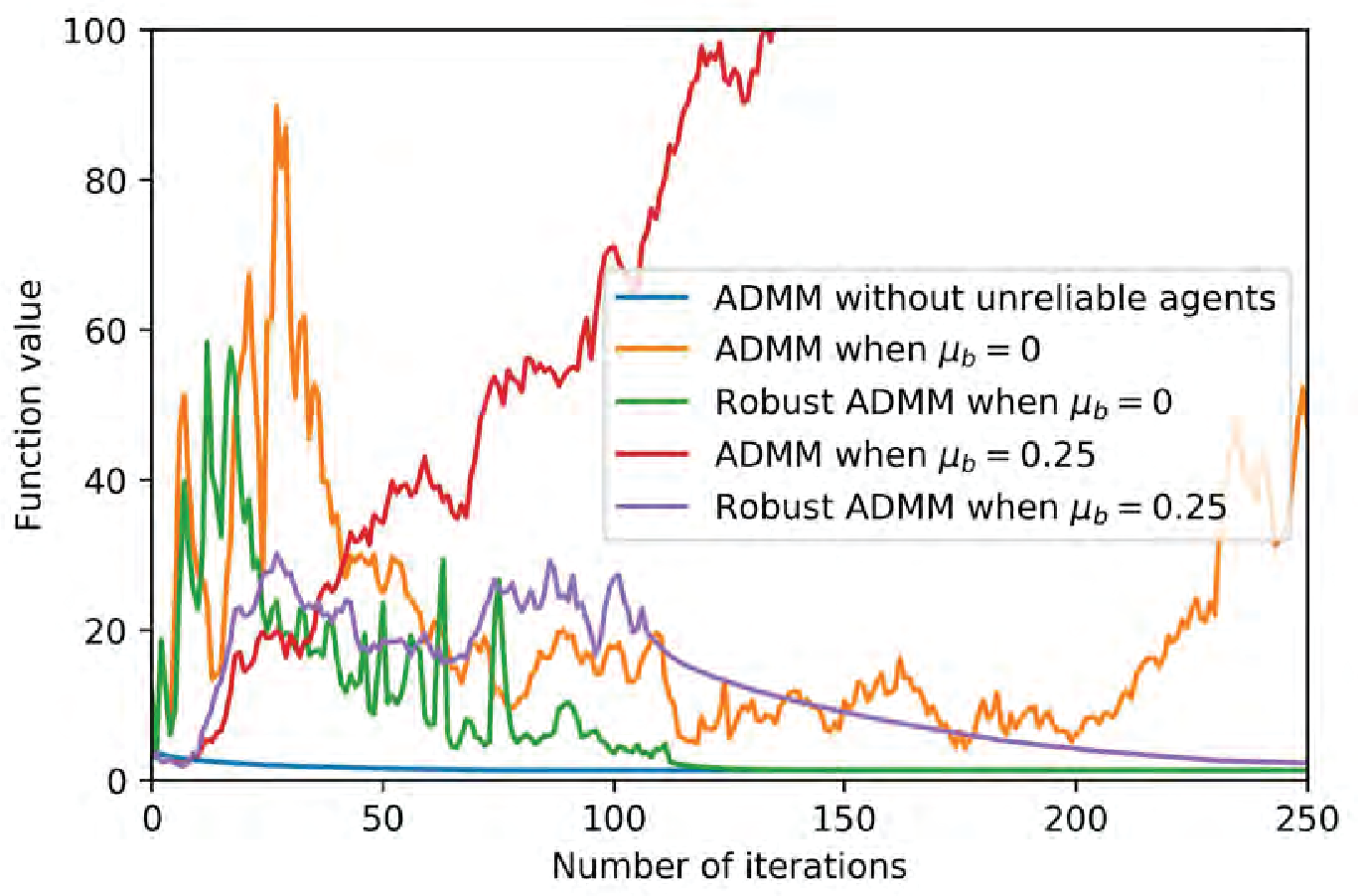} }}%
	\qquad
	\subfloat[]{{\includegraphics[width=6cm]{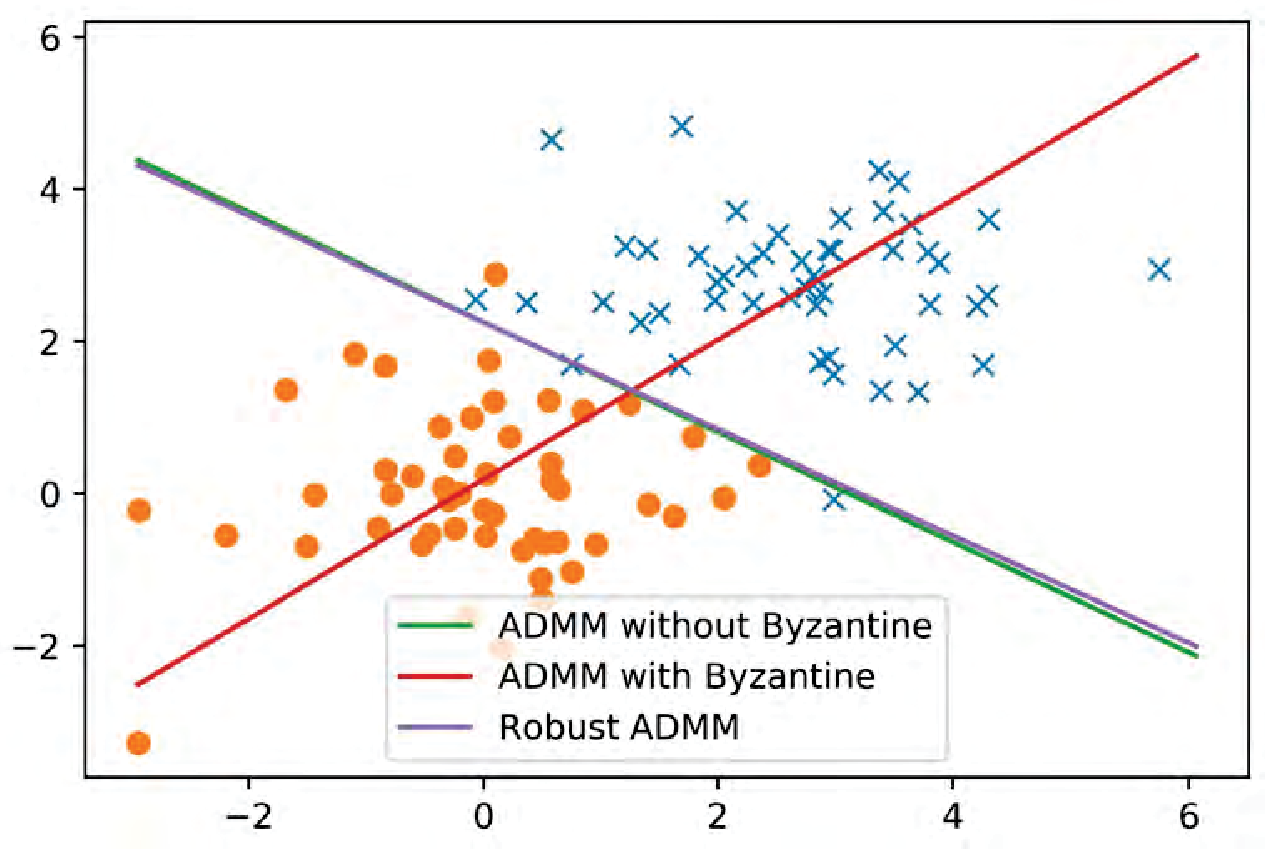} }}%
	\caption{(a) Performance comparison with different noise intensities. (b) Classification with unreliable agents.}%
	\label{fig2}%
\end{figure}
%\begin{figure}[h]
%	\vspace*{-0.05in}
%	\centerline{
%		\begin{tabular}{cccc}
%			\includegraphics[width=.25\textwidth]{a.png}\hspace*{-0.1in}  &
%			\includegraphics[width=.25\textwidth]{b.png} \hspace*{-0.1in} &
%			\includegraphics[width=.25\textwidth]{3.png} \hspace*{-0.125in} &
%			\includegraphics[width=.25\textwidth]{4.png} 
%			\\
%			\footnotesize{(a) } &   \footnotesize{(b) } & \footnotesize{(c) } &  \footnotesize{(d)} 
%	\end{tabular}}
%	\caption{\footnotesize{(a) Performance comparison with different noise intensities. (b) Performance comparison with different choices of algorithm parameter. (c) Performance comparison with different noise intensities. (d) Classification with unreliable agents.}}
%	\label{fig: exp}
%	\vspace*{-0.05in}
%\end{figure}

\subsection{Decentralized Regression}

First, we present the experimental results for a decentralized linear regression problem. The algorithm is deployed to minimize the following mean square error, 
\[\min_{\x=\tilde{\x} \in \RR^3}\sum_{i=1}^{D}\frac{1}{2}\|\y_i-\B_i\x\|^2_2.\]
For comparison, we use the same experiment setting as that in \cite{shi2014linear}. Here $\x \in \RR^3$ is the parameter to be estimated and it is generated by normal distribution $\mathcal N (\mathbf 0, \I)$, $\B_i\in \RR^{3\times3}$ is the measurement matrix of node $i$ and its elements follow $\mathcal N(0,1)$, and $\y_i \in \RR^3$ is the linear measurement vector, which is, however, corrupted by Gaussian noise $\mathcal N (\mathbf 0, \I)$. Note that the cost function is strongly convex and $L$-smooth, and we find that condition \eqref{condition_strong_convex} is satisfied by the network.
We record the cost function value over different iterations.
%\begin{figure}[h]
%	\label{fig:1}
%  \centering
%      \includegraphics[width=0.49\textwidth]{./plots/1.png}
%  \caption{Performance comparison with different noise intensities.}
%\end{figure}

In Figure $1$(a), we compare the performance of original ADMM~\cite{shi2014linear} with \textsf{ROAD} in the presence of unreliable agents. %We also report the results for the case when all the agents in the network behave reliably to provide the gold standard.
We set the noise variance for unreliable agents as $\sigma_b^2=1.5^2$ and give results for different noise intensities, i.e., $\mu_b$. We can see that if there are no unreliable agents in the network, the ADMM converges quickly to the minimizer. However, in the presence of unreliable agents, with $\mu_b=0.5$ and $\mu_b=1$, it can be seen that the performance of the original ADMM degrades significantly. We observe that original ADMM approaches a neighborhood of the minimizer whose size depends on the intensities ($\mu_b$) of the noise.
On the other hand, \textsf{ROAD} achieves a comparable convergence speed as ADMM without error.

%\begin{figure}[h]
%	\label{fig:2}
%	\centering
%	\includegraphics[width=0.49\textwidth]{./plots/acceleration.png}
%	\caption{Performance comparison with different choices of algorithm parameter.}
%\end{figure}

Next, we employ the derived optimal choice of the algorithm parameter $c$ and show the performance comparison. The optimal $c$, which is termed as $c_{opt}$, is given in Theorem \ref{convergence}. We compare the performance of the \textsf{ROAD} in the cases where $c=0.9$ and $c=c_{opt}$. We can see clearly from Figure $1$(b) that with the optimal $c$, \textsf{ROAD} achieves a much faster convergence speed. Even though the optimal algorithm parameter is derived for the situation where there are unreliable nodes, the original ADMM can also obtain an acceleration with the optimal $c$.

\subsection{Decentralized Classification}

Consider a binary classification problem with a support vector machine, and the local cost function is 
%\[\min_{\wb=\tilde{\wb} \in \RR^2, b \in \RR}\frac{1}{2}\|\wb\|^2_2+C\sum_{i=1}^{N}\max(0,1-y_i(\wb^T\x_i+b)).\]
\[f_i(\wb_i,b_i)=\frac{1}{2}\|\wb_i\|^2_2+C\sum_{j=1}^{N}\max(0,1-y_j(\wb_i^T\x_j+b_i)).\]
Here, the training set with $N=1000$ sample points is equally partitioned across $10$ agents. %, thus, every agent in the network has $100$ training samples $\{\x_j,y_j\}, 1\le j \le 100$. 
For each training point $\{\x_j,y_j\}$, $\x_j\in \RR^2$ is the feature vector, and $y_j\in \{-1,1\}$ is the corresponding label. We assume that $\x_j$ follows a normal distribution $\mathcal N([2.8,2.8]^T, \I)$ when $y_j=1$, and $\mathcal{N}(\mathbf{0},\I)$ when $y_j=-1$, respectively. Locally, the training data is evenly composed of samples from two different distributions. In our experiment, each agent updates% its own private classifier, i.e., 
$\{\wb,b\}$, and the whole network tries to reach a final consensus on a globally optimal solution.
We choose the regularization parameter $c=0.35$ in our experiment.%, and set $\lambda$ to a small value as $10^{-7}$. Thus, the formulation focuses on the hard-margin SVM problem.
 We model the error injected by unreliable agents with distribution $\mathcal{N}(0,1.5^2)$. 
%\begin{figure}[h]
%	\label{fig:3}
%	\centering
%	\includegraphics[width=0.49\textwidth]{./plots/3.png}
%	\caption{Performance comparison with unreliable agents.}
%\end{figure}

In Figure $2$(a), we present the objective function value against the number of iterations for different algorithms. We observe that in the absence of unreliable agents, the original ADMM algorithm converges quickly and there are no function value fluctuations. When unreliable agents provide erroneous updates, ADMM algorithm diverges from the minimizer significantly. %Recall Remark \ref{remark:1} and we see that since the value of the noise/error can be unbounded, the ADMM algorithm diverges from the minimizer significantly in the presence of unreliable agents. 
 We can see that when the noise intensity $\mu_b$ is larger, the size of the neighborhood is larger. On the other hand, when \textsf{ROAD} is employed, we observe that the algorithm converges to the minimizer which corroborates our theoretical results in Theorem $5$.
%\begin{figure}[h]
%	\label{fig:4}
%	\centering
%	\includegraphics[width=0.49\textwidth]{./plots/4.png}
%	\caption{Classification with unreliable agents.}
%\end{figure}

We show the classification results by depicting the hyperplane ($\wb^T \x+b=0$) in Figure $2$(b). When there are unreliable agents, the algorithm learns an ``incorrect'' classifier as is shown by the red line. By using \textsf{ROAD}, we obtain a classifier which is almost the same as the case where there are no unreliable agents. The slight difference arises because the algorithms stop after the same number of iterations in our experiments, thus, \textsf{ROAD} does not achieve the same accuracy as error-free ADMM. 

%!TEX root = main.tex
%\shrk
\section{Conclusion}\label{sec: conclude}
%\shrk
We considered the problem of decentralized learning using ADMM in the presence of unreliable agents. We studied the convergence behavior of the decentralized ADMM algorithm and showed that the ADMM converges to a neighborhood of the solution under certain conditions. We suggested guidelines for network structure design to achieve faster convergence. We also gave several conditions on the errors to obtain exact convergence to the solution. A robust variant of the ADMM algorithm was proposed to enable decentralized learning in the presence of unreliable agents and its convergence to the optima was proved. We also provided experimental results to validate the analysis and showed the effectiveness of the proposed robust scheme. We assumed the convexity of the cost function, and one might follow our lines of analysis for non-convex functions. Extension of the analysis and the algorithm to an asynchronous setting can also be considered.

\medskip
\clearpage

\bibliography{ref}
\bibliographystyle{plain}

\newpage\onecolumn
\appendix
%!TEX root = main.tex

\onecolumn
\begin{center}
	\Large \textbf{Supplementary Materials}
\end{center}

\begin{figure}[h]
	\label{fig:topol}
	\centering
	\includegraphics[width=0.45\textwidth]{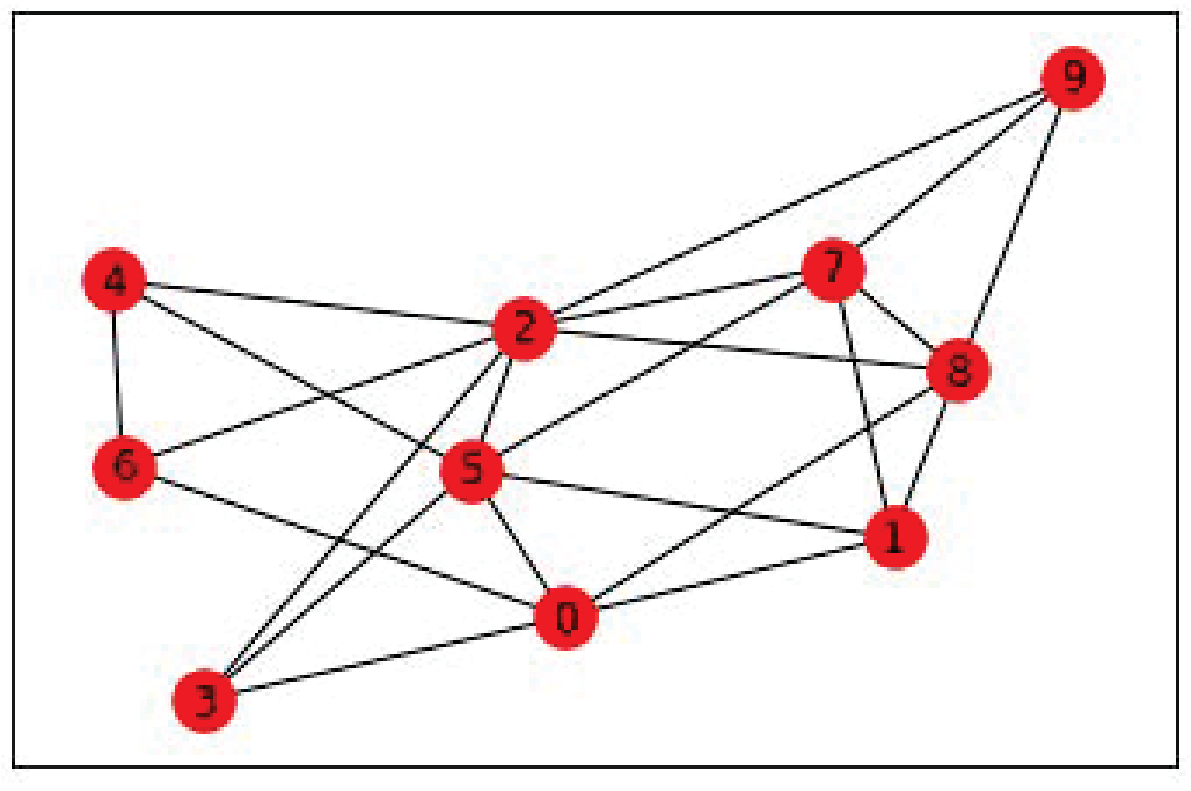}
	\caption{Decentralized network topology.}
\end{figure}

\begin{lemma}\label{lemma1}
	The update of the the algorithm can be written as
	\begin{align}
	\x^{k+1}=-\frac{1}{2c}\W^{-1} \nabla f(\x^{k+1})+\frac{\W^{-1}\Lb_+}{2}(\x^k+\e^k)-\frac{\W^{-1}\Lb_-}{2}\left(\sum\limits^k_{s=0} \x^s+\e^s \right).
	\end{align}
\end{lemma}
\begin{proof}
	Using the second step of the algorithm, we can write
	\begin{align}
	\alpha^{k+1}=\alpha^k+c\Lb_-(\x^{k+1}+\e^{k+1})
	\end{align}
	and
	\begin{align}\label{alpha_update}
	\alpha^k=\alpha^{k-1}+c\Lb_-(\x^{k}+\e^k).
	\end{align}

	Sum and telescope from iteration 0 to $k$ using \eqref{alpha_update}, and we can get the following by assuming $\alpha^0=0$
	\begin{align}
	\alpha^k=c\Lb_-\sum \limits_{s=0}^k\left(\x^s+\e^s \right).
	\end{align}
	
	Substitute the above result to the first step in the algorithm and it yields
	\begin{align}
	2c\W\x^{k+1}=-\nabla f(\x^{k+1})+c\Lb_+(\x^k+\e^k)-c\Lb_-\sum \limits_{s=0}^k\left(\x^s+\e^s \right),
	\end{align}
	which completes the proof.
\end{proof}

\begin{lemma}\label{lemma2}
	The sequences satisfy
	\begin{align}
	\frac{\Lb_+}{2}(\z^{k+1}-\z^k)-\W\e^{k+1}=-\Q\mathbf r^{k+1}-\frac{1}{2c}\nabla f(\x^{k+1})
	\end{align}
\end{lemma}
\begin{proof}
	Based on Lemma \ref{lemma1} and the fact $\W=\frac{1}{2}(\Lb_-+\Lb_+)$, we can write
	\begin{align}
	\W(\x^{k+1}-\x^k-\e^k)+\W(\x^k+\e^k)-\frac{\Lb_+}{2}(\x^k+\e^k)=-\Q\mathbf r^k-\frac{1}{2c}\nabla f(\x^{k+1}).
	\end{align}
	
	Subtracting $\frac{\Lb_-}{2}(\x^{k+1}+\e^{k+1})$ from both sides of the above equation provides
	\begin{align}
	\W(\x^{k+1}-\x^k-\e^k)+\frac{\Lb_-}{2}(\x^k+\e^k)-\frac{\Lb_-}{2}(\x^{k+1}+\e^{k+1})=-\Q\mathbf r^{k+1}-\frac{1}{2c}\nabla f(\x^{k+1}).
	\end{align}
	
	Rearrange and we have the desired result.
\end{proof}

\begin{lemma}\label{lemma3}
	The null space of $\Q$ null$(\Q)$ is span$\{\mathbf{1}\}$.
\end{lemma}
\begin{proof}
	Note that the null space of $\Q$ and $\Lb_-$ are the same. By definition, $\Lb_-=\frac{1}{2}\M_-\M_-^T$ and $\M_-=\A_1^T-\A_2^T$. Recall that if $(i,j)\in \mathcal A$ and $\y_{ij}$ is the $q$th
	block of $\y$, then the $(q,i)$th block of $\A_1$ and the $(q,j)$th block of $\A_2$ are $N\times N$ identity matrices $\I_N$; otherwise the corresponding blocks are $N\times N$ zero matrices $0_N$. Therefore, $\M_-^T=\A_1-\A_2$ is a matrix that each row has one ``1'', one ``-1'', and all zeros otherwise, which means $\M_-^T\mathbf{1}=\mathbf{0}$, i.e., null$(\M_-^T)$=span$\{\mathbf{1}\}$.
	
	Note that $\Lb_-=\frac{1}{2}\M_-\M_-^T$ and $\Q=\left(\frac{\Lb_-}{2}\right)^{\frac{1}{2}}$, thus null$(\Q)$=null$(\M_-^T)$, completing the proof.
\end{proof}

\begin{lemma}\label{lemma4}
	For some $\mathbf r^\ast$ that satisfies $\Q\mathbf r^\ast+\frac{1}{2c}\nabla f(\x^\ast)=0$ and $\mathbf r^\ast$ belongs to the column space of $\Q$, the sequences satisfy
	\begin{align}
	\frac{\Lb_+}{2}(\z^{k+1}-z^k)-\W\e^{k+1}=-\Q(\mathbf r^{k+1}-\mathbf r^\ast)-\frac{1}{2c}(\nabla f(\x^{k+1})-\nabla f(\x^\ast))
	\end{align}
\end{lemma}
\begin{proof}
	Using Lemma \ref{lemma2}, we have 
	\begin{align}
	\frac{\Lb_+}{2}(\z^{k+1}-\z^k)-\W\e^{k+1}=-\Q\mathbf r^{k+1}-\frac{1}{2c}\nabla f(\x^{k+1}).
	\end{align}
	
	According to Lemma \ref{lemma3}, null$(\Q)$ is span$\{\mathbf{1}\}$. Since $\mathbf{1}^T\nabla f(\x^\ast)=0$, $\nabla f(\x^\ast)$ can be written as a linear combination of column vectors of $\Q$. Therefore, there exists $\mathbf r$ such that $\frac{1}{2c}\nabla f(x^\ast)=-\Q\mathbf r$. Let $\mathbf r^\ast$ be the projection of $\mathbf r$ onto $\Q$ to obtain $\Q\mathbf r=\Q\mathbf r^\ast$ where $\mathbf r^\ast$ lies in the column space of $\Q$.
	
	Hence, we can write
	\begin{align}
	\frac{\mathbf L_+}{2}(\z^{k+1}-\z^k)-\W\e^{k+1}=-\Q(\mathbf r^{k+1}-\mathbf r^\ast)-\frac{1}{2c}(\nabla f(\x^{k+1})-\nabla f(\x^\ast))
	\end{align}
\end{proof}

\begin{lemma}\label{lemma5}
	$\langle \x^\ast, \Q \rangle =0.$
\end{lemma}
\begin{proof}
	Since the optimal consensus solution $\x^\ast$ has an identical value for all its entries, $\x^\ast$ lies in the space spanned by $\mathbf{1}$. Thus, according to Lemma \ref{lemma3}, we have the desired result, and also $\langle \x^\ast, \Lb_- \rangle =0.$
\end{proof}
\section{Proof of Theorem \ref{convex}}
\begin{proof}
We prove the first part in Theorem \ref{convex}.
	Assuming $f(\x)$ is convex, we can have
	\begin{align}
		f(\x^{k+1})-f(\x^\ast) \le \langle \x^{k+1}-\x^\ast, \nabla f(x^{k+1}) \rangle.
	\end{align}
 By Lemma \ref{lemma2}, it yields
 \begin{align}
 	f(\x^{k+1})-f(\x^\ast) \le& \langle \x^{k+1}-\x^\ast, 2c\W \e^{k+1}-2c\Q\rb^{k+1}-c\Lb_+(\z^{k+1}-\z^k) \rangle \\
 	 =&\langle \x^{k+1}-\x^\ast, c\Lb_+(\z^{k}-\z^{k+1}) \rangle +\langle \x^{k+1}-\x^\ast, 2c\W \e^{k+1} \rangle \\
     &+\langle \x^{k+1}-\x^\ast, -2c\Q\rb^{k+1} \rangle \\
 	 =&\langle \z^{k+1}-\z^\ast, c\Lb_+(\z^{k}-\z^{k+1}) \rangle-\langle \e^{k+1},c\Lb_+(\z^{k}-\z^{k+1}) \rangle \\&+\langle \z^{k+1}-\z^\ast, -2c\Q\rb^{k+1} \rangle-\langle \e^{k+1},-2c\Q\rb^{k+1} \rangle +\langle \e^{k+1}, 2c\W (\x^{k+1}-\x^\ast) \rangle\\
 	 =& 2\langle \z^{k+1}-\z^\ast, \frac{c\Lb_+}{2}(\z^k-\z^{k+1})\rangle +2\langle \rb^k-\rb^{k+1},c(\rb^{k+1}-\rb^\prime)\rangle\\
 	 &+\langle  \e^{k+1}, c\Lb_+(\z^{k+1}-\z^k)+2c\Q\rb^{k+1}+2c\W(\x^{k+1}-\x^\ast)\rangle .\label{inter_conv_pre}
 \end{align}
 
 If the algorithm stops at $T$-th iteration, then the function value $f(\x^T)$ is affected by the error $\e^k$ with $k=0,1,\ldots,T-1$. Thus, we can set $k=T-1$ and $\e^T=\mathbf 0$ in the above bound, and obtain
 \begin{align}
 		f(\x^{T})-f(\x^\ast) \le& \|\z^{T-1}-\z^\ast\|^2_{\frac{c\Lb_+}{2}}-\|\z^{T}-\z^\ast\|^2_{\frac{c\Lb_+}{2}}-\|\z^{T-1}-\z^{T}\|^2_{\frac{c\Lb_+}{2}}\\
 		&+c\|\rb^{T-1}-\rb^\prime\|^2_2-c\|\rb^{T}-\rb^\prime\|^2_2-c\|\rb^{T-1}-\rb^{T}\|^2_2\\
 		\le &\|\q^{T-1}-\p\|^2_{\G}.
 \end{align}

Now we prove the second part in Theorem \ref{convex}.	By convexity, for any $\rb \in \RR^{DN}$, we can have
	\begin{align}
		&\frac{f(\x^{k+1})-f(\x^\ast)}{c}+2\rb^\prime\Q\x^{k+1}\\
		\le& \langle \x^{k+1}-\x^\ast, -\Lb_+(\x^{k+1}-\x^k)-\Lb_+(\x^k-\z^k)-2\Q(\rb^{k+1}-\rb)+\Lb_-(\z^{k+1}-\x^{k+1})\rangle\\
		=&\langle \x^{k+1}-\x^\ast, \Lb_+(\x^{k}-\x^{k+1})\rangle+\langle \x^{k+1}-\x^\ast,\Lb_+(\z^k-\x^k)\rangle+\langle\x^{k+1}-\x^\ast,2\Q(\rb-\rb^{k+1}) \rangle\\
		&+\langle\x^{k+1}-\x^\ast,\Lb_-(\z^{k+1}-\x^{k+1}) \rangle\\
		=&\langle \x^{k+1}-\x^\ast, \Lb_+(\x^{k}-\x^{k+1})\rangle+\langle \x^{k+1}-\x^\ast,\Lb_+(\z^k-\x^k)\rangle+\langle\z^{k+1}-\x^\ast,2\Q(\rb-\rb^{k+1}) \rangle\\
		&+\langle\x^{k+1}-\x^\ast,\Lb_-(\z^{k+1}-\x^{k+1}) \rangle+\langle \e^{k+1},2\Q(\rb^{k+1}-\rb)\rangle\\
		=&\langle \x^{k+1}-\x^\ast, \Lb_+(\x^{k}-\x^{k+1})\rangle+\langle \x^{k+1}-\x^\ast,\Lb_+(\z^k-\x^k)\rangle+\langle\rb^{k+1}-\rb^k,2(\rb-\rb^{k+1}) \rangle\\
		&+\langle\x^{k+1}-\x^\ast,\Lb_-(\z^{k+1}-\x^{k+1}) \rangle+\langle \e^{k+1},2\Q(\rb^{k+1}-\rb)\rangle\\
		=&\frac{1}{c}(\|\p^k-\p\|^2_\G-\|\p^{k+1}-\p\|^2_\G-\|\p^{k+1}-\p^k\|^2_\G)+\langle \x^{k+1}-\x^\ast,\Lb_+(\z^k-\x^k)\rangle\\
		&+\langle\x^{k+1}-\x^\ast,\Lb_-(\z^{k+1}-\x^{k+1}) \rangle+\langle \e^{k+1},2\Q(\rb^{k+1}-\rb)\rangle\\
		=&\frac{1}{c}(\|\p^k-\p\|^2_\G-\|\p^{k+1}-\p\|^2_\G)-\|\Q\x^{k+1}\|_2^2-\|\Q\e^{k+1}\|_2^2+2\langle \frac{\Lb_+}{2}(\x^{k+1}-\x^\ast),\z^k-\x^k\rangle\\
		&+\langle \e^{k+1},2\Q(\rb^{k+1}-\rb)\rangle\\
		=&\frac{1}{c}(\|\p^k-\p\|^2_\G-\|\p^{k+1}-\p\|^2_\G)-\frac{\sigma_{\min}(\Lb_-)}{2}\|\x^{k+1}-\x^\ast\|_2^2-\|\Q\e^{k+1}\|_2^2\\
		&+\frac{1}{\alpha}\|\frac{\Lb_+}{2}(\x^{k+1}-\x^\ast)\|^2_2+\alpha\|\z^k-\x^k\|^2_2+\langle \e^{k+1},2\Q(\rb^{k+1}-\rb)\rangle\\
		\myeq&\quad\quad\quad \frac{1}{c}(\|\p^k-\p\|^2_\G-\|\p^{k+1}-\p\|^2_\G)-\|\Q\e^{k+1}\|_2^2+\frac{\sigma^2_{\max}(\Lb_+)}{2\sigma_{\min}(\Lb_-)} \|\z^k-\x^k\|^2_2\\
		&+\langle \e^{k+1},2\Q(\rb^{k+1}-\rb)\rangle\\
% 		=& \frac{1}{c}(\|\p^k-\p\|^2_\G-\|\p^{k+1}-\p\|^2_\G)+\frac{\sigma^2_{\max}(\Lb_+)}{2\sigma_{\min}(\Lb_-)} \|\e^k\|^2_2+\|\rb^{k+1}-\rb\|^2_2.
\le& \frac{1}{c}(\|\p^k-\p\|^2_\G-\|\p^{k+1}-\p\|^2_\G) +\frac{\sigma^2_{\max}(\Lb_+)}{2\sigma_{\min}(\Lb_-)} \|\e^k\|^2_2+\langle \e^{k+1},2\Q(\rb^{k+1}-\rb)\rangle.
	\end{align}
By letting $\rb=0$, telescope and sum from $k=0$ to $T-1$ (the error for the last iteration $\e^T=0$), and we obtain
\begin{align}
\frac{1}{c}\sum_{k=1}^T\left( f(\x^k)-f(\x^\ast)\right) \le \frac{1}{c}\|\p^0-\p\|^2_\G+\sum_{k=0}^{T-1} \left(\frac{\sigma^2_{\max}(\Lb_+)}{2\sigma_{\min}(\Lb_-)} \|\e^{k}\|^2_2+\langle \e^{k},2\Q\rb^{k})\rangle \right).
\end{align}
		Rearrange and we have the desired result.	
\end{proof}

\section{Proof of Theorem \ref{lemma6}}
\begin{proof}
By $v$-strong convexity, we obtain
	\begin{align}\label{str_conv}
	{v}\|\x^{k+1}-\x^\ast\|_2^2 \le&  \langle \x^{k+1}-\x^\ast,\nabla f(\x^{k+1})-\nabla f(\x^\ast) \rangle \\= & \langle \x^{k+1}-\x^\ast, c\mathbf L_+(\z^k-\z^{k+1})+2c\W\e^{k+1}-2c\Q(\mathbf r^{k+1}-\mathbf r^\ast)\rangle \\
	=& \langle \x^{k+1}-\x^\ast, c\mathbf L_+(\z^k-\z^{k+1})\rangle + \langle \x^{k+1}-\x^\ast, 2c\W\e^{k+1}\rangle\\
	&+
	\langle \x^{k+1}-\x^\ast, -2c\Q(\mathbf r^{k+1}-\mathbf r^\ast)\rangle \\
	=& \langle \z^{k+1}-\z^\ast, c\mathbf L_+(\z^k-\z^{k+1})\rangle- \langle \e^{k+1}, c\mathbf L_+(\z^k-\z^{k+1})\rangle\\
	&+ \langle \x^{k+1}+\e^{k+1}-\x^\ast, -2c\Q(\mathbf r^{k+1}-\mathbf r^\ast)\rangle \\ &-\langle \e^{k+1}, -2c\Q(\mathbf r^{k+1}-\mathbf r^\ast)\rangle+ \langle \e^{k+1}, 2c\W(\x^{k+1}-\x^\ast )\rangle\\
	=& 2\langle \z^{k+1}-\z^\ast, \frac{c\mathbf L_+}{2}(\z^k-\z^{k+1}) \rangle +2\langle \mathbf r^k-\mathbf r^{k+1}, c (\mathbf r^{k+1}-\mathbf r^\ast)\rangle\\
	&+\langle \e^{k+1}, c\mathbf L_+(\z^{k+1}-\z^k)+2c\Q(\mathbf r^{k+1}-\mathbf r^\ast)+2c\W(\x^{k+1}-\x^\ast)\rangle \\
	=& \|\q^k-\q^\ast\|^2_\G-\|\q^{k+1}-\q^\ast\|^2_\G-\|\q^k-\q^{k+1}\|^2_\G\\
	&+\langle \e^{k+1}, c\mathbf L_+(\z^{k+1}-\z^k)+2c\Q(\mathbf r^{k+1}-\mathbf r^\ast)+2c\W(\x^{k+1}-\x^\ast)\rangle
	\end{align}
	
	For any $\lambda>0$, using the basic inequality
	\begin{align}
	\|\mathbf a+\mathbf b\|^2_2+(\lambda-1)\|\mathbf a\|_2^2 \ge (1-\frac{1}{\lambda})\|\mathbf b\|_2^2
	\end{align}
	we can write for $\lambda_1> 1$ and $\lambda_2> 1$
	\begin{align}
	&\frac{\sigma_{\max}^2(\mathbf L_+)}{4}\|\z^{k+1}-\z^k\|^2_2+\frac{(\lambda_1-1)L^2\|\x^{k+1}-\x^\ast\|^2_2}{4c^2}\\
	&\ge \|\frac{\mathbf L_+}{2}(\z^{k+1}-\z^k)\|^2_2+(\lambda_1-1)\|\frac{1}{2c}\left(\nabla f(\x^{k+1})-\nabla f(\x^\ast) \right)\|^2_2\\
	&\ge \left( 1-\frac{1}{\lambda_1}\right)\|\W\e^{k+1}-\Q(\mathbf r^{k+1}-\mathbf r^\ast)\|^2_2\\
	& \ge \left( 1-\frac{1}{\lambda_1}\right)(1-\frac{1}{\lambda_2})\|\Q(\mathbf r^{k+1}-\mathbf r^\ast)\|^2_2-\left( 1-\frac{1}{\lambda_1}\right)(\lambda_2-1)\|\W\e^{k+1}\|^2_2\\
	&\ge \left( 1-\frac{1}{\lambda_1}\right)(1-\frac{1}{\lambda_2})\sigma^2_{\min}(\Q)\|\mathbf r^{k+1}-\mathbf r^\ast\|^2_2-\left( 1-\frac{1}{\lambda_1}\right)(\lambda_2-1)\sigma^2_{\max}(\W)\|\e^{k+1}\|^2_2.
	\end{align}
	Thus, for a positive quantity $\delta$,
	\begin{align}\label{int_delta}
	&\frac{\delta\sigma_{\max}^2(\mathbf L_+)\lambda_1\lambda_2}{4\sigma^2_{\min}(\mathbf Q)(\lambda_1-1)(\lambda_2-1)}\|\z^{k+1}-\z^k\|^2_2+\frac{\delta\lambda_1\lambda_2 L^2\|\x^{k+1}-\x^\ast\|^2_2}{4c^2\sigma^2_{\min}(\Q)(\lambda_2-1)}\\
	& \ge \delta \|\mathbf r^{k+1}-\mathbf r^\ast\|^2_2-\frac{\delta\lambda_2\sigma^2_{\max}(\W)}{\sigma^2_{\min}(\Q)}\|\e^{k+1}\|^2_2.
	\end{align}
	
	Since $\x^{k+1}-\x^\ast=\z^{k+1}-\z^\ast-\e^{k+1}$, 
	%using the basic inequality $\|a+b\|^2_2\le (\lambda+1)\|a\|_2^2+\left(1+\frac{1}{\lambda}\right)\|b\|^2_2$ 
	for any $\lambda_3 >1$, we can get
	\begin{align}\label{xz_inq}
	\|\x^{k+1}-\x^\ast\|^2_2 \ge \left(1-\frac{1}{\lambda_3}\right)\|\z^{k+1}-\z^\ast\|^2_2-\left(\lambda_3-1\right)\|\e^{k+1}\|^2_2.
	\end{align}

	Therefore, the addition of \eqref{int_delta}$\times c^2$ and \eqref{xz_inq}$\times \frac{\delta c^2\sigma^2_{\max}(L_+)\lambda_3}{4(\lambda_3-1)}$ yields
	\begin{align}
	&\frac{c^2\delta\sigma_{\max}^2(\mathbf L_+)\lambda_1\lambda_2}{4\sigma^2_{\min}(\Q)(\lambda_1-1)(\lambda_2-1)}\|\z^{k+1}-\z^k\|^2_2+\left(\frac{\delta\lambda_1\lambda_2 L^2}{4\sigma^2_{\min}(\Q)(\lambda_2-1)}+\frac{\delta c^2\sigma^2_{\max}(\mathbf L_+)\lambda_3}{4(\lambda_3-1)}\right)\|\x^{k+1}-\x^\ast\|^2_2\\
	& \ge \delta c^2 \|\mathbf r^{k+1}-\mathbf r^\ast\|^2_2+\frac{\delta c^2\sigma^2_{\max}(\mathbf L_+)}{4}\|\z^{k+1}-\z^k\|^2_2-\left(\frac{c^2 \delta\lambda_2\sigma^2_{\max}(\W)}{\sigma^2_{\min}(\Q)}+\frac{\delta c^2\sigma^2_{\max}(\mathbf L_+)\lambda_3}{4}\right)\|\e^{k+1}\|^2_2\\
	& \ge \delta \|c(\mathbf r^{k+1}-\mathbf r^\ast)\|^2_2+\delta \|\frac{c\mathbf L_+}{2}(\z^{k+1}-\z^k)\|^2_2 -\left(\frac{c^2 \delta\lambda_2\sigma^2_{\max}(\W)}{\sigma^2_{\min}(\Q)}+\frac{\delta c^2\sigma^2_{\max}(\mathbf L_+)\lambda_3}{4}\right)\|\e^{k+1}\|^2_2\\
	& =\delta \|\q^{k+1}-\q^\ast\|^2_\G-\left(\frac{c^2 \delta\lambda_2\sigma^2_{\max}(\W)}{\sigma^2_{\min}(\Q)}+\frac{\delta c^2\sigma^2_{\max}(\mathbf L_+)\lambda_3}{4}\right)\|\e^{k+1}\|^2_2.
	\end{align}
	
	Choose $\delta$ to be such that 
	\begin{align}
	\frac{c^2\delta\sigma_{\max}^2(\mathbf L_+)\lambda_1\lambda_2}{4\sigma^2_{\min}(\Q)(\lambda_1-1)(\lambda_2-1)} &\le \frac{c^2\sigma^2_{\min}\mathbf (\mathbf L_+)}{4}\\
	\left(\frac{\delta\lambda_1\lambda_2 L^2}{4\sigma^2_{\min}(\mathbf Q)(\lambda_2-1)}+\frac{ \delta c^2\sigma^2_{\max}(\mathbf L_+)\lambda_3}{4(\lambda_3-1)}\right) &\le v,
	\end{align}
	and we can have
	\begin{align}
	&\frac{c^2\sigma^2_{\min}(\mathbf L_+)}{4}\|\z^{k+1}-\z^k\|^2_2+v \|\x^{k+1}-\x^k\|^2_2 \ge\\ 
	&\delta \|\q^{k+1}-\q^\ast\|^2_\G-\left(\frac{c^2 \delta\lambda_2\sigma^2_{\max}(\W)}{\sigma^2_{\min}(\Q)}+\frac{\delta c^2\sigma^2_{\max}(\mathbf L_+)\lambda_3}{4}\right)\|\e^{k+1}\|^2_2.
	\end{align}
	Thus, it is straightforward to write
	\begin{align}
	& \|\q^{k+1}-\q^k\|_\G^2+v \|\x^{k+1}-\x^k\|^2_2\\
	&\ge \|c(\mathbf r^{k+1}-\mathbf r^k)\|^2_2+\|\frac{c\mathbf L_+}{2}(\z^{k+1}-\z^k)\|^2_2+v \|\x^{k+1}-\x^k\|^2_2\\
	&\ge c^2\|\mathbf r^{k+1}-\mathbf r^k\|^2_2 +\frac{c^2\sigma^2_{\min}(\mathbf L_+)}{4}\|\z^{k+1}-\z^k\|^2_2+v \|\x^{k+1}-\x^k\|^2_2 \\
	&\ge c^2\|\mathbf r^{k+1}-\mathbf r^k\|^2_2+\delta \|\q^{k+1}-\q^\ast\|^2_\G-\left(\frac{c^2 \delta\lambda_2\sigma^2_{\max}(\W)}{\sigma^2_{\min}(\Q)}+\frac{\delta c^2\sigma^2_{\max}(\mathbf L_+)\lambda_3}{4}\right)\|\e^{k+1}\|^2_2.
	\end{align}
	Recall the result in \eqref{str_conv} regarding the bound to $v \|\x^{k+1}-\x^k\|^2_2$, and we can further write
	\begin{align}
	&\|\q^k-\q^\ast\|^2_\G-\|\q^{k+1}-\q^\ast\|^2_\G\\
	&+\langle \e^{k+1}, c\mathbf L_+(\z^{k+1}-\z^k)+2c\Q(\mathbf r^{k+1}-\mathbf r^\ast)+2c\W(\x^{k+1}-\x^\ast)\rangle \\
	& \ge \delta \|\q^{k+1}-\q^\ast\|^2_\G-\left(\frac{c^2 \delta\lambda_2\sigma^2_{\max}(\W)}{\sigma^2_{\min}(\Q)}+\frac{\delta c^2\sigma^2_{\max}(\mathbf L_+)\lambda_3}{4}\right)\|\e^{k+1}\|^2_2
	\end{align}
	
	Let $P=\frac{c^2 \delta\lambda_2\sigma^2_{\max}(\W)}{\sigma^2_{\min}(\Q)}+\frac{\delta c^2\sigma^2_{\max}(\mathbf L_+)\lambda_3}{4}$. Rearrange the expression and we get
	\begin{align}
	\|\q^{k+1}-\q^\ast\|^2_\G \le & \frac{\|\q^{k}-\q^\ast\|^2_\G}{1+\delta}+\frac{P}{1+\delta}\|\e^{k+1}\|^2_2\\
	&+\frac{1}{1+\delta}\langle \e^{k+1}, c\mathbf L_+(\z^{k+1}-\z^k)+2c\Q(\mathbf r^{k+1}-\mathbf r^\ast)+2c\W(\x^{k+1}-\x^\ast)\rangle
	\end{align}
\end{proof}

\begin{lemma}\label{lemma8}
	Let $\beta \in (0, \frac{1+\delta}{4})$, $b \in (0,1)$, $\lambda_4 >1$, and then we have
	\begin{align}
	&(1-b)\left(\frac{1}{4}-\frac{\beta}{1+\delta}\right)\sigma_{\min}^2(\mathbf L_+)\|\z^{k+1}-\z^\ast\|^2_2+\left(1-\frac{4\beta}{1+\delta}\right)\|\mathbf r^{k+1}-\mathbf r^\ast\|^2_2\\
	&+b\left(\frac{1}{4}-\frac{\beta}{1+\delta}\right)\sigma_{\min}^2(\mathbf L_+)\left(1-\frac{1}{\lambda_4}\right)\|\x^{k+1}-\x^\ast\|^2_2\\
	&\le \frac{1/4+\beta}{1+\delta}\sigma^2_{\max}(\mathbf L_+)\|\z^{k}-\z^\ast\|^2_2+ \frac{1}{1+\delta}\|\mathbf r^{k}-\mathbf r^\ast\|^2_2\\
	&+\left[\frac{P+1/2\beta}{(1+\delta)c^2}+b\left(\frac{1}{4}-\frac{\beta}{1+\delta}\right)\sigma_{\min}^2(\mathbf L_+)(\lambda_4-1)\right] \|\e^{k+1}\|^2_2\\
	&+\frac{4\beta\sigma^2_{\max}(\W)}{1+\delta}\|\x^{k+1}-\x^\ast\|^2_2.
	\end{align}
\end{lemma}
\begin{proof}
	First, we rewrite the result in Lemma \ref{lemma6} in the following form
	\begin{align}
	&\|\frac{c\mathbf L_+}{2}(\z^{k+1}-\z^\ast)\|^2_2+\|c(\mathbf r^{k+1}-\mathbf r^\ast)\|^2_2 \le \frac{1}{1+\delta}\left(\|\frac{c\mathbf L_+}{2}(\z^{k}-\z^\ast)\|^2_2+\|c(\mathbf r^{k}-\mathbf r^\ast)\|^2_2 \right) \\
	&+\frac{P}{1+\delta}\|\e^{k+1}\|^2_2+\frac{1}{1+\sigma} \langle \e^{k+1}, c\mathbf L_+(\z^{k+1}-\z^k)+2c\Q(\mathbf r^{k+1}-\mathbf r^\ast)+2c\W(\x^{k+1}-\x^\ast) \rangle\\
	& \le \frac{1}{1+\delta}\left(\|\frac{c\mathbf L_+}{2}(\z^{k}-\z^\ast)\|^2_2+\|c(\mathbf r^{k}-\mathbf r^\ast)\|^2_2 \right)+\left(\frac{P}{1+\delta}+\frac{1/2\beta}{1+\delta} \right) \|\e^{k+1}\|^2_2\\
	&+\frac{\beta}{1+\delta}\|\frac{c\mathbf L_+}{2}(\z^{k+1}-\z^k)+c\Q(\mathbf r^{k+1}-\mathbf r^\ast)+c\W(\x^{k+1}-\x^\ast) \|^2_2\\
	& \le \frac{1}{1+\delta}\left(\|\frac{c\mathbf L_+}{2}(\z^{k}-\z^\ast)\|^2_2+\|c(\mathbf r^{k}-\mathbf r^\ast)\|^2_2 \right)+\left(\frac{P}{1+\delta}+\frac{1/2\beta}{1+\delta} \right) \|\e^{k+1}\|^2_2\\
	&+\frac{\beta}{1+\delta}\|\frac{c\mathbf L_+}{2}(\z^{k+1}-\z^k)+c\Q(\mathbf r^{k+1}-\mathbf r^\ast)+c\W(\x^{k+1}-\x^\ast) \|^2_2\\
	& \le \frac{1}{1+\delta}\left(\|\frac{c\mathbf L_+}{2}(\z^{k}-\z^\ast)\|^2_2+\|c(\mathbf r^{k}-\mathbf r^\ast)\|^2_2 \right)+\left(\frac{P}{1+\delta}+\frac{1/2\beta}{1+\delta} \right) \|\e^{k+1}\|^2_2\\
	&+\frac{4\beta}{1+\delta}\|\frac{c\Lb_+}{2}(\z^{k+1}-\z^\ast)\|^2_2+\frac{4\alpha}{1+\delta}\|\frac{c\mathbf L_+}{2}(\z^{k}-\z^\ast)\|^2_2\\
	&+\frac{4\beta}{1+\delta}\|c\Q(\mathbf r^{k+1}-\mathbf r^\ast)\|^2_2+\frac{4\beta}{1+\delta}\|c\W(\x^{k+1}-\x^\ast)\|^2_2
	\end{align}
	where $\beta >0$.
	
	Rearranging the inequality provides
	\begin{align}
	&\left(1-\frac{4\beta}{1+\delta}\right)\|\frac{c\mathbf L_+}{2}(\z^{k+1}-\z^\ast)\|^2_2+\left(1-\frac{4\beta}{1+\delta}\right)\|c(\mathbf r^{k+1}-\mathbf r^\ast)\|^2_2\\
	&\le \left( \frac{1}{1+\delta}+\frac{4\beta}{1+\delta}\right)\|\frac{c\mathbf L_+}{2}(\z^{k}-\z^\ast)\|^2_2+ \frac{1}{1+\delta}\|c(\mathbf r^{k}-\mathbf r^\ast)\|^2_2\\
	&+\frac{P+1/2\beta}{1+\delta} \|\e^{k+1}\|^2_2+\frac{4\beta}{1+\delta}\|c\W(\x^{k+1}-\x^\ast)\|^2_2.
	\end{align}
	Note that the parameters should be chosen such that $\left(1-\frac{4\beta}{1+\delta}\right) >0$.
	
	Then we can write
	\begin{align}
	&\left(\frac{1}{4}-\frac{\beta}{1+\delta}\right)\sigma_{\min}^2(\mathbf L_+)\|\z^{k+1}-\z^\ast\|^2_2+\left(1-\frac{4\beta}{1+\delta}\right)\|\mathbf r^{k+1}-\mathbf r^\ast\|^2_2\\
	&\le \left( \frac{1}{4(1+\delta)}+\frac{\beta}{1+\delta}\right)\sigma^2_{\max}(\mathbf L_+)\|\z^{k}-\z^\ast\|^2_2+ \frac{1}{1+\delta}\|\mathbf r^{k}-\mathbf r^\ast\|^2_2\\
	&+\frac{P+1/2\beta}{(1+\delta)c^2} \|\e^{k+1}\|^2_2+\frac{4\beta\sigma^2_{\max}(\W)}{1+\delta}\|\x^{k+1}-\x^\ast\|^2_2.
	\end{align}
	
	Since we have the inequality $\|\z^{k+1}-\z^\ast\|^2_2 \ge \left(1-\frac{1}{\lambda_4}\right)\|\x^{k+1}-\x^\ast\|^2_2-(\lambda_4-1)\|\e^{k+1}\|^2_2$, for $b\in (0,1)$, we can get
	\begin{align}
	b\left(\frac{1}{4}-\frac{\beta}{1+\delta}\right)\sigma_{\min}^2(\mathbf L_+)\|\z^{k+1}-\z^\ast\|^2_2 \ge& b\left(\frac{1}{4}-\frac{\beta}{1+\delta}\right)\sigma_{\min}^2(\mathbf L_+)\left(1-\frac{1}{\lambda_4}\right)\|\x^{k+1}-\x^\ast\|^2_2\\
	&-b\left(\frac{1}{4}-\frac{\beta}{1+\delta}\right)\sigma_{\min}^2(\mathbf L_+)(\lambda_4-1)\|\e^{k+1}\|^2_2.
	\end{align}
	Thus,
	\begin{align}
	&(1-b)\left(\frac{1}{4}-\frac{\beta}{1+\delta}\right)\sigma_{\min}^2(\mathbf L_+)\|\z^{k+1}-\z^\ast\|^2_2+\left(1-\frac{4\beta}{1+\delta}\right)\|\mathbf r^{k+1}-\mathbf r^\ast\|^2_2\\
	&+b\left(\frac{1}{4}-\frac{\beta}{1+\delta}\right)\sigma_{\min}^2(\mathbf L_+)\left(1-\frac{1}{\lambda_4}\right)\|\x^{k+1}-\x^\ast\|^2_2\\
	&\le \frac{1/4+\beta}{1+\delta}\sigma^2_{\max}(\mathbf L_+)\|\z^{k}-\z^\ast\|^2_2+ \frac{1}{1+\delta}\|\mathbf r^{k}-\mathbf r^\ast\|^2_2\\
	&+\left[\frac{P+1/2\beta}{(1+\delta)c^2}+b\left(\frac{1}{4}-\frac{\beta}{1+\delta}\right)\sigma_{\min}^2(\mathbf L_+)(\lambda_4-1)\right] \|\e^{k+1}\|^2_2\\
	&+\frac{4\beta\sigma^2_{\max}(\W)}{1+\delta}\|\x^{k+1}-\x^\ast\|^2_2.
	\end{align}
\end{proof}

Defining
\begin{align}
A_1=\frac{4}{(1-b)\sigma^2_{\min}(\mathbf L_+)},
\end{align}
and
\begin{align}
A_2=\frac{4}{(1+4\beta)\sigma^2_{\max}(\mathbf L_+)},
\end{align}
we have the desired result.

\section{Proof of Theorem \ref{thm1}}

	\subsection{Eliminate $\|\x^{k+1}-\x^\ast\|^2_2$}
	First, we want to eliminate the term $\|\x^{k+1}-\x^\ast\|^2_2$ in Lemma \ref{lemma8}, which requires
	\begin{align}
	 b\left(\frac{1}{4}-\frac{\beta}{1+\delta}\right)\sigma_{\min}^2(\mathbf L_+)\left(1-\frac{1}{\lambda_4}\right) \ge \frac{4\beta\sigma^2_{\max}(\W)}{1+\delta}
	\end{align}
	and it is equivalent to that
	\begin{align}
	\beta \le \frac{b(1+\delta)\sigma^2_{\min}(\mathbf L_+)\left(1-\frac{1}{\lambda_4}\right)}{4b\sigma^2_{\min}(\mathbf L_+)\left(1-\frac{1}{\lambda_4}\right)+16\sigma^2_{\max}(\W)}
	\end{align}
	
	Then we can write
	\begin{align}
	&(1-b)\left(\frac{1}{4}-\frac{\beta}{1+\delta}\right)\sigma_{\min}^2(\mathbf L_+)\|\z^{k+1}-\z^\ast\|^2_2+\left(1-\frac{4\beta}{1+\delta}\right)\|\mathbf r^{k+1}-\mathbf r^\ast\|^2_2\\
	&\le \frac{1/4+\beta}{1+\delta}\sigma^2_{\max}(\mathbf L_+)\|\z^{k}-\z^\ast\|^2_2+ \frac{1}{1+\delta}\|\mathbf r^{k}-\mathbf r^\ast\|^2_2\\
	&+\left[\frac{P+1/2\beta}{(1+\delta)c^2}+b\left(\frac{1}{4}-\frac{\beta}{1+\delta}\right)\sigma_{\min}^2(\mathbf L_+)(\lambda_4-1)\right] \|\e^{k+1}\|^2_2\\
	\end{align}
	which can be further simplified
	\begin{align}\label{96}
	&\|\z^{k+1}-\z^\ast\|^2_2+A_1\|\mathbf r^{k+1}-\mathbf r^\ast\|^2_2
	\le B(\|\z^{k}-\z^\ast\|^2_2+ A_2\|\mathbf r^{k}-\mathbf r^\ast\|^2_2)
	+C\|\e^{k+1}\|^2_2
	\end{align}
	We require the following for convergence analysis
	\begin{align}
	A_1 \ge A_2
	\end{align}
	which leads to the requirement
	\begin{align}
	(1-b)\sigma^2_{\min}(\mathbf L_+) \le (1+\beta)\sigma^2_{\max}(\mathbf L_+).
	\end{align}
	Note that this requirement is satisfied intrinsically.
	
	Therefore, we get
	\begin{align}
	\|\z^{k+1}-\z^\ast\|^2_2+A_1\|\mathbf r^{k+1}-\mathbf r^\ast\|^2_2\le  B^{k+1} \left( \|\z^{0}-\z^\ast\|^2_2+A_2\|\mathbf r^{0}-\mathbf r^\ast\|^2_2+\sum\limits^{k+1}_{s=1} B^{-s}C\|\e^{s}\|^2_2\right)
	\end{align}
	and we have the desired result since $A_1\|r^{k+1}-r^\ast\|^2_2 >0$.

\subsection{$B\in (0,1)$}
The above convergence result requires that $B\in (0,1)$. First, having $\beta$ in Theorem \ref{thm1} at hand, we can make sure that $B$ is greater than 0. Then, it requires that $B<1$ and correspondingly
\begin{align}
{(1+4\beta)\sigma^2_{\max}(\mathbf L_+)}\le {(1-b)(1+\delta-4\beta)\sigma^2_{\min}(\mathbf L_+)}
\end{align}
which is equivalent to that
\begin{align}
&\beta \le \frac{(1-b)(1+\delta)\sigma^2_{\min}(\mathbf L_+)-\sigma^2_{\max}(\mathbf L_+)}{4\sigma^2_{\max}(\mathbf L_+)+4(1-b)\sigma^2_{\min}(\mathbf L_+)}
\end{align}
and
\begin{align}
{(1-b)(1+\delta)\sigma^2_{\min}(\mathbf L_+)-\sigma^2_{\max}(\mathbf L_+)} >0.
\end{align}

Since $b$ can be arbitrarily chosen from  $ (0,1)$, we also need
\begin{align}
 0<\frac{\sigma^2_{\max}(\mathbf L_+)}{(1+\delta)\sigma^2_{\min}(\mathbf L_+)}<1
\end{align}
One intuition is that we should design a network such that $\frac{\sigma^2_{\max}(\mathbf L_+)}{\sigma^2_{\min}(\mathbf L_+)}$ is the smallest possible. Substituting $\delta$ in the expression and we have
\begin{align}
\frac{\sigma^2_{\min}(\mathbf L_+)}{\sigma^2_{\max}(\mathbf L_+)} > \frac{L^2-2v+\sqrt{(L^2+2v)^2+16v^2\frac{\lambda_2-1}{\lambda_2}\sigma^2_{\min}(\Q)}}{4v\frac{\lambda_2-1}{\lambda_2}\sigma^2_{\min}(\Q)+2L^2}.
\end{align}

\section{Proof of Theorem \ref{convergence}}
	Note that $\delta$ is chosen as
	\begin{align}
	\delta=\min\left\lbrace\frac{(\lambda_1-1)(\lambda_2-1)\sigma^2_{\min}(\Q)\sigma^2_{\min}(\mathbf L_+)}{\lambda_1\lambda_2\sigma^2_{\max}(\mathbf L_+)},\frac{4v(\lambda_2-1)(\lambda_3-1)\sigma^2_{\min}(\Q)}{\lambda_1\lambda_2(\lambda_3-1)L^2+c^2\lambda_3(\lambda_2-1)\sigma^2_{\max}(\mathbf L_+)\sigma^2_{\min}(\Q)}\right\rbrace
	\end{align}
	We choose $c$ such that
	\begin{align}
	\lambda_1\lambda_2(\lambda_3-1)L^2=c^2\lambda_3(\lambda_2-1)\sigma^2_{\max}(\mathbf L_+)\sigma^2_{\min}(\Q),
	\end{align}
	which yields
	\begin{align}
	c=\sqrt[]{\frac{\lambda_1\lambda_2(\lambda_3-1)L^2}{\lambda_3(\lambda_2-1)\sigma^2_{\max}(\mathbf L_+)\sigma^2_{\min}(\Q)}}
	\end{align}
	and
	\begin{align}
	\delta=&\min\left\lbrace\frac{(\lambda_1-1)(\lambda_2-1)\sigma^2_{\min}(\Q)\sigma^2_{\min}(\mathbf L_+)}{\lambda_1\lambda_2\sigma^2_{\max}(\mathbf L_+)},\frac{2v(\lambda_2-1)\sigma^2_{\min}(\Q)}{\lambda_1\lambda_2L^2}\right\rbrace \\
	=& \frac{(\lambda_2-1)\sigma^2_{\min}(\Q)}{\lambda_2}\min \left\lbrace  \frac{(\lambda_1-1)\sigma^2_{\min}(\mathbf L_+)}{\lambda_1\sigma^2_{\max}(\mathbf L_+)},\frac{2v}{\lambda_1L^2} \right\rbrace
	\end{align}
	
	It is desirable that $\delta$ can achieve its maximum, which is obtained by
	\begin{align}
	\frac{(\lambda_1-1)\sigma^2_{\min}(\mathbf L_+)}{\lambda_1\sigma^2_{\max}(\mathbf L_+)}=\frac{2v}{\lambda_1L^2} .
	\end{align}
	Therefore, we can set $\lambda_1$ as
	\begin{align}
	\lambda_1=1+\frac{2v\sigma^2_{\max}(\mathbf L_+)}{L^2\sigma^2_{\min}(\mathbf L_+)},
	\end{align}
	and thus, we have $\delta$ as
	\begin{align}
	\delta=\frac{(\lambda_2-1)}{\lambda_2}\frac{2v\sigma^2_{\min}(\Q)\sigma^2_{\min}(\mathbf L_+)}{L^2\sigma^2_{\min}(\mathbf L_+)+2v\sigma^2_{\max}(\mathbf L_+)}
	\end{align}

	The constraint on $\beta$ in Theorem 4 ensures that $B>0$.

	Note that $\lambda_3$ only appears in $C$ and $P$. It is straightforward to derive the optimal $\lambda_3$ to minimize $C$, and we arrive at
	\begin{align}
	\lambda_3=\sqrt[]{\frac{L^2\sigma^2_{\min}(\mathbf L_+)+2v\sigma^2_{\max}(\mathbf L_+)}{\beta\lambda_1L^2v\sigma^2_{\min}(\mathbf L_+)}}+1
	\end{align}
	thus resulting in
	\begin{align}
	C=\frac{\frac{4\delta \lambda_2 \sigma^2_{\max}(\W)}{\sigma^2_{\min}(\Q)}+\sigma^2_{\max}(\mathbf L_+)\left(\sqrt[]{\delta}+\sqrt[]{\frac{2(\lambda_2-1)\sigma^2_{\min}(\Q)}{\beta \lambda_1\lambda_2L^2}}\right)^2}{(1-b)(1+\delta)(1+\delta-4\beta)\sigma^2_{\min}(\mathbf L_+)}+\frac{b(\lambda_4-1)}{1-b}
	\end{align}

\section{Proof of Corollary \ref{strongly_convex_corollary}}
\subsection{First one:}
	According to the result in Theorem \ref{thm1}, we have
	\begin{align}
	\|\z^{k+1}-\z^\ast\|^2_2 &\le B^{k+1}(\|\z^{0}-\z^\ast\|^2_2+ A_1\|\mathbf r^{0}-\mathbf r^\ast\|^2_2)+B^{k+1}\sum\limits^{k+1}_{s=1} B^{-s}C\|\e^{s}\|^2_2
	\end{align}
	and then
	\begin{align}
	\|\z^{k+1}-\z^\ast\|^2_2 & \le B^{k+1}(\|\z^{0}-\z^\ast\|^2_2+ A_1\|\mathbf r^{0}-\mathbf r^\ast\|^2_2)+Ce B^{k+1}\sum\limits^{k+1}_{s=1} B^{-s}\\
	& = B^{k+1}(\|\z^{0}-\z^\ast\|^2_2+ A_1\|\mathbf r^{0}-\mathbf r^\ast\|^2_2)+Ce\frac{1-B^{k+1}}{1-B}\\
	& \le B^{k+1}(\|\z^{0}-\z^\ast\|^2_2+ A_1\|\mathbf r^{0}-\mathbf r^\ast\|^2_2)+\frac{Ce}{1-B}.
	\end{align}
	
	Since $B\in (0,1)$, we have the desired result.
\subsection{Second one:}
	Recall the result in \eqref{96},
	\begin{align}
	\|\z^{k+1}-\z^\ast\|^2_2& \le B^{k+1}(\|\z^{0}-\z^\ast\|^2_2+ A_1\|\mathbf r^{0}-\mathbf r^\ast\|^2_2)+B^{k+1}\sum\limits^{k+1}_{s=1} B^{-s}C\|\e^{s}\|^2_2
	\end{align}
	which then can be written as
	\begin{align}
	\|\z^{k+1}-\z^\ast\|^2_2& \le B^{k+1}(\|\z^{0}-\z^\ast\|^2_2+ A_1\|\mathbf r^{0}-\mathbf r^\ast\|^2_2)+B^{k+1}C\sum\limits^{k+1}_{s=1} B^{-s}R^s\|\e^{0}\|^2_2 \\
	& \le B^{k+1}(\|z^{0}-\z^\ast\|^2_2+ A_1\|\mathbf r^{0}-\mathbf r^\ast\|^2_2)+B^{k+1}C\|\e^{0}\|^2_2\sum\limits^{k+1}_{s=1} \left(\frac{R}{B}\right)^{s}\\
	& \le B^{k+1}(\|\z^{0}-\z^\ast\|^2_2+ A_1\|\mathbf r^{0}-\mathbf r^\ast\|^2_2)+B^{k+1}C\|\e^{0}\|^2_2\frac{R}{B-R}\\
	& = B^{k+1}(\|\z^{0}-\z^\ast\|^2_2+ A_1\|\mathbf r^{0}-\mathbf r^\ast\|^2_2+\frac{RC\|\e^{0}\|^2_2}{B-R})
	\end{align}
	completing the proof.
\subsection{Third one:}
Recall the result in \eqref{96},
\begin{align}
&\|\z^{k+1}-\z^\ast\|^2_2+A_1\|\mathbf r^{k+1}-\mathbf r^\ast\|^2_2
\le B(\|\z^{k}-\z^\ast\|^2_2+ A_2\|\mathbf r^{k}-\mathbf r^\ast\|^2_2)
+C\|\e^{k+1}\|^2_2.
\end{align}
If 
$
C\|\e^{k+1}\|^2_2 \le B(A_1-A_2)\|\mathbf r^k-\mathbf r^\ast\|^2_2,
$
we can write
\begin{align}
\|\z^{k+1}-\z^\ast\|^2_2+A_1\|\mathbf r^{k+1}-\mathbf r^\ast\|^2_2
&\le B(\|\z^{k}-\z^\ast\|^2_2+ A_2\|\mathbf r^{k}-\mathbf r^\ast\|^2_2)
+C\|\e^{k+1}\|^2_2\\
& \le B(\|\z^{k}-\z^\ast\|^2_2+ A_2\|\mathbf r^{k}-\mathbf r^\ast\|^2_2)+ B(A_1-A_2)\|\mathbf r^k-\mathbf r^\ast\|^2_2 \\
& \le B(\|\z^{k}-\z^\ast\|^2_2+ A_1\|\mathbf r^{k}-\mathbf r^\ast\|^2_2).
\end{align}
Then we have
\begin{align}
\|\z^{k+1}-\z^\ast\|^2_2+A_1\|\mathbf r^{k+1}-\mathbf r^\ast\|^2_2
\le B^{k+1}(\|\z^{0}-\z^\ast\|^2_2+ A_1\|\mathbf r^{0}-\mathbf r^\ast\|^2_2),
\end{align}
which leads to
\begin{align}
\|\z^{k+1}-\z^\ast\|^2_2
\le B^{k+1}(\|\z^{0}-\z^\ast\|^2_2+ A_1\|\mathbf r^{0}-\mathbf r^\ast\|^2_2),
\end{align}
completing the proof as $B\in (0,1)$.

\section*{}
\begin{lemma}\label{abs}
There exists a vector $\y\in \RR^N$ and $\sigma_{\min}(\mathbf \y\y^T)=1$, such that $\forall \x\in \RR^N$, $\y^T\x \ge \|\x\|$.
\end{lemma}
\begin{proof}
Since $\forall \x\in \RR^D$, $\y^T\x \ge \|\x\|$, it leads to
\begin{align}
\x^T\y\y^T\x\ge \x^T\x,
\end{align}
which is equivalent to 
\begin{align}
\sigma_{\min}(\mathbf \y\y^T)=1.
\end{align}
\end{proof}

\begin{lemma}\label{U}
In the error-free case, starting from $\x^0=0$, we have
\begin{align}
\frac{1}{T}\sum_{k=1}^T\|\Q\x^k\|\le \frac{1}{4T}\left( \sigma_{\max}(\Lb_+)V^2_1+\frac{2V_2^2}{\sigma_{\min}(\Lb_-)c^2}+4\right).
\end{align}
\end{lemma}
\begin{proof}
First, for any $\rb\in \RR^{DN}$, we obtain
\begin{align}
	&\frac{f(\x^{k+1})-f(\x^\ast)}{c}+\langle2\Q\rb, \x^{k+1} \rangle \\
	=& \langle \x^{k+1}-\x^\ast, -\Lb_+(\x^{k+1}-\x^k)-2\Q(\rb^{k+1}-\rb)  \rangle\\
	=&\langle \x^{k+1}-\x^\ast, -\Lb_+(\x^{k+1}-\x^k)-2\Q(\rb^{k+1}-\rb) \rangle\\
    =&\langle \x^{k+1}-\x^\ast,-\Lb_+(\x^{k+1}-\x^k) \rangle+\langle \rb^{k+1}-\rb^k,-2(\rb^{k+1}-\rb)\rangle.
\end{align}
Telescope and sum from $k=0,\ldots, T$, we can get
\begin{align}
&\frac{1}{c}\sum_{k=1}^{T}{f(\x^{k})-f(\x^\ast)}+2\rb^\prime\Q\x^{k} \\
\le& \|\x^0-\x^\ast\|^2_{\frac{\Lb_+}{2}}-\|\x^T-\x^\ast\|^2_{\frac{\Lb_+}{2}}-\sum_{k=1}^T\|\x^k-\x^{k-1}\|^2_{\frac{\Lb_+}{2}}\\
&+\|\rb^0-\rb\|^2_2-\|\rb^T-\rb\|^2_2-\sum_{k=1}^T\|\rb^k-\rb^{k-1}\|^2_2。
\end{align}
Therefore, we obtain
\begin{align}\label{induction}
\frac{1}{c}\sum_{k=0}^{T}{f(\x^{k})-f(\x^\ast)}+2\rb^\prime\Q\x^{k} \le \|\x^0-\x^\ast\|^2_{\frac{\Lb_+}{2}}+\|\rb^0-\rb\|^2_2
\end{align}

Define $\hat{\x_T}=\frac{\sum_{k=1}^T\x^k}{T}$ and we get the following by Jensen's inequality as
\begin{align}
f(\hx)-f(\x^\ast) +2c\rb^\prime\Q\hx \le \frac{c}{T}\|\p^0-\p\|_\G^2.
\end{align}
If we choose $\rb=0$, we obtain
\begin{align}
f(\hx)-f(\x^\ast) \le \frac{c}{T}\left(\|\x^0-\x^\ast\|^2_{\frac{\Lb_+}{2}}+\|\rb^0\|^2_2\right).
\end{align}
The saddle point inequality implies
\begin{align}
f(\x^\ast)-f(\hx) \le 2c\langle \Q\rb^\ast, \hx\rangle .
\end{align}
Thus, using \eqref{induction}, it yields
\begin{align}
2c\langle \Q\rb^\ast, \hx\rangle \le f(\hx)-f(\x^\ast)+ 2c\langle \Q2\rb^\ast, \hx\rangle \le \frac{c}{T}\left(\|\x^0-\x^\ast\|^2_{\frac{\Lb_+}{2}}+\|\rb^0-2\rb^\ast\|^2_2\right)
\end{align}
Now we let $\rb=\rb^\ast+\y$ with $\y$ chosen according to Lemma \ref{abs}. Thus, we obtain
\begin{align}
f(\hx)-f(\x^\ast) +2c\langle\Q\rb^\ast,\hx \rangle + 2c\y^T\Q\hx  \le \frac{c}{T}\left(\|\x^0-\x^\ast\|^2_{\frac{\Lb_+}{2}}+\|\rb^0-\rb^\ast-\y\|^2_2\right).
\end{align}
Since $(\x^\ast,\rb^\ast)$ is a primal-dual optimal solution, the saddle point inequality provides
\begin{align}
f(\hx)-f(\x^\ast) +2c\langle\Q\rb^\ast,\hx \rangle \ge 0.
\end{align}
Using Lemma \ref{abs}, we obtain
\begin{align}
\frac{2c}{T}\sum_{k=1}^T\|\Q\x^k\|\le \frac{c}{T}\left(\|\x^0-\x^\ast\|^2_{\frac{\Lb_+}{2}}+\|\rb^0-\rb^\ast-\y\|^2_2\right),
\end{align}
which yields
\begin{align}
\frac{1}{T}\sum_{k=1}^T\|\Q\x^k\|\le \frac{1}{2T}\left(\|\x^0-\x^\ast\|^2_{\frac{\Lb_+}{2}}+2\|\rb^0-\rb^\ast\|_2^2+2\right).
\end{align}

Choose the starting point $\x^0=0$ and thus $\rb^0=0$, and we have 
\begin{align}
\frac{1}{T}\sum_{k=1}^T\|\Q\x^k\|\le \frac{1}{2T}\left(\|\x^\ast\|^2_{\frac{\Lb_+}{2}}+2\|\rb^\ast\|_2^2+2\right)\le \frac{1}{4T}\left( \sigma_{\max}(\Lb_+)V^2_1+\frac{2V_2^2}{\sigma_{\min}(\Lb_-)c^2}+4\right).
\end{align}
\end{proof}

\section{Proof of Theorem \ref{radmm}}
	For any $\rb \in \RR^{DN} $, we can write
	\begin{align}
		&\frac{f(\x^{k+1})-f(\x^\ast)}{c}+2\rb^\prime\Q\x^{k+1}\\
		\le& \langle \x^{k+1}-\x^\ast, -\Lb_+(\x^{k+1}-\x^k)-\Lb_+(\x^k-\z^k)-2\Q(\rb^{k+1}-\rb)+\Lb_-(\z^{k+1}-\x^{k+1})\rangle\\
		=&\langle \x^{k+1}-\x^\ast, \Lb_+(\x^{k}-\x^{k+1})\rangle+\langle \x^{k+1}-\x^\ast,\Lb_+(\z^k-\x^k)\rangle+\langle\x^{k+1}-\x^\ast,2\Q(\rb-\rb^{k+1}) \rangle\\
		&+\langle\x^{k+1}-\x^\ast,\Lb_-(\z^{k+1}-\x^{k+1}) \rangle\\
		=&\langle \x^{k+1}-\x^\ast, \Lb_+(\x^{k}-\x^{k+1})\rangle+\langle \x^{k+1}-\x^\ast,\Lb_+(\z^k-\x^k)\rangle+\langle\z^{k+1}-\x^\ast,2\Q(\rb-\rb^{k+1}) \rangle\\
		&+\langle\x^{k+1}-\x^\ast,\Lb_-(\z^{k+1}-\x^{k+1}) \rangle+\langle \e^{k+1},2\Q(\rb^{k+1}-\rb)\rangle\\
		=&\langle \x^{k+1}-\x^\ast, \Lb_+(\x^{k}-\x^{k+1})\rangle+\langle \x^{k+1}-\x^\ast,\Lb_+(\z^k-\x^k)\rangle+\langle\rb^{k+1}-\rb^k,2(\rb-\rb^{k+1}) \rangle\\
		&+\langle\x^{k+1}-\x^\ast,\Lb_-(\z^{k+1}-\x^{k+1}) \rangle+\langle \e^{k+1},2\Q(\rb^{k+1}-\rb)\rangle\\
		=&\frac{1}{c}(\|\p^k-\p\|^2_\G-\|\p^{k+1}-\p\|^2_\G-\|\p^{k+1}-\p^k\|^2_\G)+\langle \x^{k+1}-\x^\ast,\Lb_+(\z^k-\x^k)\rangle\\
		&+\langle\x^{k+1}-\x^\ast,\Lb_-(\z^{k+1}-\x^{k+1}) \rangle+\langle \e^{k+1},2\Q(\rb^{k+1}-\rb)\rangle\\
		=&\frac{1}{c}(\|\p^k-\p\|^2_\G-\|\p^{k+1}-\p\|^2_\G)-\|\Q\x^{k+1}\|_2^2-\|\Q\e^{k+1}\|_2^2+2\langle \frac{\Lb_+}{2}(\x^{k+1}-\x^\ast),\z^k-\x^k\rangle\\
		&+\langle \e^{k+1},2\Q(\rb^{k+1}-\rb)\rangle\\
		=&\frac{1}{c}(\|\p^k-\p\|^2_\G-\|\p^{k+1}-\p\|^2_\G)-\frac{\sigma_{\min}(\Lb_-)}{2}\|\x^{k+1}-\x^\ast\|_2^2-\|\Q\e^{k+1}\|_2^2\\
		&+\frac{1}{\alpha}\|\frac{\Lb_+}{2}(\x^{k+1}-\x^\ast)\|^2_2+\alpha\|\z^k-\x^k\|^2_2+\langle \e^{k+1},2\Q(\rb^{k+1}-\rb)\rangle\\
		\myeq&\quad\quad\quad \frac{1}{c}(\|\p^k-\p\|^2_\G-\|\p^{k+1}-\p\|^2_\G)-\|\Q\e^{k+1}\|_2^2+\frac{\sigma^2_{\max}(\Lb_+)}{2\sigma_{\min}(\Lb_-)} \|\z^k-\x^k\|^2_2\\
		&+\langle \e^{k+1},2\Q(\rb^{k+1}-\rb)\rangle\\
		=& \frac{1}{c}(\|\p^k-\p\|^2_\G-\|\p^{k+1}-\p\|^2_\G)-\|\Q\e^{k+1}\|_2^2+\frac{\sigma^2_{\max}(\Lb_+)}{2\sigma_{\min}(\Lb_-)} \|\z^k-\x^k\|^2_2\\
		&+\|2\Q \e^{k+1}\|\|\rb^{k+1}-\rb\|
	\end{align}
    
    Algorithm \textsf{ROAD} guarantees that $\sum_{t=1}^k \|\Q\z^t\| \le 2EU/\sqrt[]{2}=\sqrt[]{2}EU$, and $\sum_{t=1}^k \|\Q\X^t\| \le \sqrt[]{2}EU$ due to the thresholding as well. Thus, we have $\sum_{t=1}^k \|\Q\e^t\| \le 2\sqrt{2}EU$. Then, we can have
    \begin{align}
    \frac{f(\x^{k+1})-f(\x^\ast)}{c}+2\rb^\prime\Q\x^{k+1}
		\le & \frac{1}{c}(\|\p^k-\p\|^2_\G-\|\p^{k+1}-\p\|^2_\G)-\|\Q\e^{k+1}\|_2^2\\
        &+\frac{\sigma^2_{\max}(\Lb_+)}{\sigma^2_{\min}(\Lb_-)} \|\Q\e^k\|^2_2+\|2\Q \e^{k+1}\|(\sqrt{2}EU+\|\rb\|).
    \end{align}
	Telescope and sum from $k=0$ to $T-1$ ($\e^{T}=0$ since it is the last iteration), and we get
	\begin{align}
\sum_{k=1}^{T}{f(\x^{k})-f(\x^\ast)}+2c\rb^\prime\Q\x^{k} &\le \|\p^0-\p\|^2_\G-\|\p^{T}-\p\|^2_\G+2c\sum_{k=1}^T\|\Q\e^k\|(\sqrt{2}EU+\|\rb\|)\\
&+c\frac{\sigma^2_{\max}(\Lb_+)-\sigma_{\min}^2(\Lb_-)}{\sigma^2_{\min}(\Lb_-)}\sum_{k=1}^T\|\Q\e^k\|^2_2\\
&\le \|\p^0-\p\|^2_\G-\|\p^{T}-\p\|^2_\G+c4\sqrt{2}EU(\sqrt{2}EU+\|\rb\|)\\
&+c\frac{\sigma^2_{\max}(\Lb_+)-\sigma_{\min}^2(\Lb_-)}{\sigma^2_{\min}(\Lb_-)}8E^2U^2\\
&=\|\p^0-\p\|^2_\G-\|\p^{T}-\p\|^2_\G+c4\sqrt{2}EU\|\rb\|\\
&+c\frac{\sigma^2_{\max}(\Lb_+)}{\sigma^2_{\min}(\Lb_-)}8E^2U^2.
	\end{align}

Choosing $\rb=0$, we obtain
\begin{align}
	f(\hat{\x}_T)-f(\x^\ast) \le \frac{1}{T}\left( \|\p^0-\p\|^2_\G+8c\frac{\sigma^2_{\max}(\Lb_+)}{\sigma^2_{\min}(\Lb_-)}E^2U^2\right).
\end{align}

\end{document}